\documentclass[11pt]{article}

\usepackage{mathrsfs,amsmath,amsfonts,amssymb,amstext,amscd,bm,bbm,dsfont,pifont, 
            amsthm,stmaryrd,euscript,color,xcolor,accents,xr} 
\usepackage[backref]{hyperref}
\usepackage{algorithmic}
\usepackage{algorithm}
\usepackage{url}
\usepackage{mathtools}
\usepackage{epsfig}
\usepackage{float}
\usepackage{appendix}
\usepackage{amstext}
\usepackage{floatflt}
\usepackage{nicefrac}
\usepackage{amsmath}
\usepackage{anysize,hyperref}
\usepackage{enumerate}
\usepackage{epsfig}
\usepackage{hyperref}
\usepackage{graphicx}
\usepackage{mathtools}
\usepackage{ upgreek }
\usepackage{pdfpages}
\numberwithin{equation}{section} 
\input{macros_nicole}
\newcommand{\heavy}{\theta} 
\usepackage[linecolor=magenta!60!, backgroundcolor=magenta!10!,textwidth=1.6cm, textcolor=magenta]{todonotes}

\title{How many Neurons do we need? \\ A refined Analysis for Shallow Networks \\ trained with Gradient Descent}

\author{Mike Nguyen\\
TU Braunschweig \\ 
\texttt{mike.nguyen@tu-braunschweig.de}
\and 
Nicole M\"ucke \\
TU Braunschweig \\
\texttt{nicole.muecke@tu-braunschweig.de} 
}
\date{\today}

\begin{document}

\maketitle

\begin{abstract}

We analyze the generalization properties of two-layer neural networks in the 
neural tangent kernel (NTK) regime, trained with gradient descent (GD). 
For early stopped GD  
we derive fast rates of convergence that are known to be minimax optimal 
in the framework of non-parametric regression in reproducing kernel Hilbert spaces.  
On our way, we precisely keep track of the number of hidden neurons 
required for generalization and improve over existing results. 
We further show that the weights during training 
remain in a vicinity around initialization, the radius being dependent on structural assumptions such as degree of 
smoothness of the regression function and eigenvalue decay of the integral operator associated to the NTK.       
\end{abstract}

{\bf Keywords:} Neural Tangent Kernel $\bullet$ Early Stopping  $\bullet$ Gradient Descent


\section{Introduction}

The rapid advancement of artificial intelligence in recent years has been largely propelled 
by the remarkable capabilities of neural networks. These computational models have revolutionized 
numerous domains, including image recognition, natural language processing, and autonomous systems. 
The effectiveness of neural networks in solving complex tasks has led to their widespread adoption 
across academia and industry. 
\\
Understanding why standard optimization algorithms often
find globally optimal solutions despite the intricate non convexity of training loss functions has
become a focal point of research. Furthermore, deep neural networks, despite their
vast number of parameters, tend to exhibit impressive generalization capabilities, achieving
high accuracy on unseen data \cite{neyshabur2019towards}. These optimization and generalization phenomena lie at
the heart of deep learning theory, presenting fundamental challenges.


In this paper, we explore the learning properties of shallow neural networks in the NTK regime, when trained 
with gradient descent. 
It is studied in \cite{lee2019wide}, how wide neural networks behave during gradient descent training. 
In the limit of infinite width, these networks can be approximated as linear models via the first-order 
Taylor expansion around their initial parameters. Moreover, \cite{jacot2018neural} established that 
training a neural network with a specific parameterization is equivalent to employing a kernel method 
as the network's width approaches infinity. It is shown that gradient flow (GF) on the 
parameter space becomes kernel GF on the function space. 
This explains why local minima of the train error become global minima in the infinite width limit, 
see also \cite{venturi2019spurious, jentzen2021existence}. 
\\
A line of research 
\cite{daniely2016toward, belkin2018understand, wei2019regularization,arora2019fine, arora2019exact, 
lee2022neural, chen2020generalized} 
explored this kernel method analogy and 
demonstrated that, with adequate over-parameterization, a certain initialization scale, and an appropriate 
learning rate schedule, gradient descent effectively learns a linear classifier 
on top of the initial random features. 
\\
The works \cite{du2019gradient, allen2019convergence, allen2019convergence, zou2020gradient} 
investigate gradient descent convergence to global minima. They demonstrate that for i.i.d. 
Gaussian initialization, wide networks experience minimal parameter changes during training. 
This is key to the phenomenon where wide neural networks exhibit linear behavior in terms 
of their parameters during training. For a survey we also refer to \cite{golikov2022neural}.

Of particular interest is establishing optimal bounds for the generalization error with a 
minimal number of neurons required. 
Compared to the number of results for the train error, 
only a few investigations can be found for generalization,  
see e.g.  \cite{song2019quadratic, lai2023generalization2, arora2019fine} for 
shallow neural networks and \cite{cao2019generalization,zhang2020type} for deep neural networks. 
\\
In \cite{E_2020}, the authors prove for GF that $O(n)$ many neurons are sufficient to obtain an 
optimal generalization bound of order $O(\sqrt{n})$. However they only trained the outer layer and had restrictive 
assumptions on the target function. The authors in \cite{nitanda2020optimal} prove fast rates of convergence 
for SGD in the case that the regression function belongs to the reproducing kernel Hilbert space (RKHS) 
associated to the NTK. However, exponentially many neurons are needed to obtain this result. 
Also closely related to our work is 
the work \cite{braun2021smoking}, analyzing the $L^2$-error of neural network regression estimates  
with one hidden layer with a logistic squasher, trained with gradient descent, under a smoothness assumption 
on the Fourier transform of the regression function. 
The authors show a rate of convergence of $\sqrt{n}$, 
up to a polylogarithmic factor after $n^{7/4}$ iterations (up to a polylog factor). 
In order to achieve this, the number of hidden neurons increases as $\sqrt{n}$. 
We refer to Table \ref{table3} for a comparison.



{\bf Our contribution.}
We improve the above results 
in different directions: 
\begin{itemize}
\item We derive an early stopping time $T_n=\cO\left(n^{\frac{1}{2r+b}}\right)$ 
that leads to minimax optimal rates of convergence. This depends  
on the smoothness $r>0$ of the regression function and the (polynomial) decay rate 
$b \in (0,1]$ of the eigenvalues of the kernel integral operator associated to the NTK. Our results hold for 
so called {\it easy learning} problems, where $2r+b > 1$.  
 
\item We present a refined number of neurons that are needed for optimality of order 
$M_n \geq O\left( n^{\frac{2r}{2r+b}}\right)=  O(T^{2r}_n)$ for $r\geq\frac{1}{2}$, i.e. for well-specified cases where 
the regression function belongs to the RKHS and  
$M_n \geq O\left( n^{\frac{3-4r}{2r+b}}\right)=  O(T^{3-4r}_n)$ for $r \in (0,\frac{1}{2})$. 
The latter includes the case 
where the regression does not necessarily belongs to the RKHS associated to the NTK. 

\item We also overcome the saturation effect appearing in \cite{nitanda2020optimal} by providing 
fast rates of convergence for smooth objectives, i.e. for $r>1$. 

\item 
Furthermore, we prove that during GD with constant step size, in the well-specified case, the weights stay bounded 
if $M\geq O(T^{2r}_n)$ (for $r \geq 1/2$), up to a logarithmic factor. If $r \leq 1/2$, 
the weights are in a ball of radius $O(T^{1/2-r}_n)$, up to a logarithmic factor. 
To the best of our knowledge, previous work only had been able to bound the weights with an decaying step size or 
with exponentially many neurons. 

\item Notably, the number of hidden neurons that are sufficient to establish our results is comparable to the number 
of random features for learning in RKHSs, 
see e.g. \cite{rudi2017generalization, carratino2018learning, nguyen2023random}. 
\end{itemize}

\begin{table}[h]
\begin{center}
\begin{tabular}{|l|l|l|l|}
\hline References & Width $M$ & Iterations $T$ & Method \\
\hline \hline \cite{E_2020} & $O\left(n\right)$ & $O\left(\sqrt{n}\right)$ & GF \\
\hline \cite{nitanda2020optimal} & $O\left(\exp(n)\right)$ & $O\left(n\right)$& SGD  \\
\hline \cite{braun2021smoking} & $O\left(\sqrt n\right)$ & $O\left(n^{7/4}\right)$& GD  \\
\hline 
\hline Our work & $O\left(\sqrt{n}\right)$ & $O\left(\sqrt{n}\right)$& GD  \\
\hline 
\end{tabular}
\caption{\label{table3} Number of neurons and iterations needed to provide a generalization 
bound of order $O(n^{-\frac{1}{2}})$. }
\end{center}
\end{table}


{\bf Organization.} 
In Section \ref{sec:setting} we define the mathematical framework needed to present our main results in  
Section \ref{sec:main-results}. 
We defer all proofs to the Appendices. 

{\bf Notation.} 
For two Hilbert spaces $\cH_1, \cH_2$ 
and a linear operator $A: \cH_1 \to \cH_2$, we write $A^*: \cH_2 \to \cH_1$ to denote the adjoint operator.  
If $\theta \in \cH_1$ we 
write $\theta \otimes \theta := \langle \cdot, \theta \ra \theta$ to denote the tensor product. 
For $n \in \mbn$, we write $[n]=\{1,...,n\}$.  For two positive sequences $(a_n)_n$, $(b_n)_n$ 
we write $a_n \lesssim b_n$ if $a_n \leq c b_n$
for some $c>0$ and $a_n \simeq b_n$ if both $a_n \lesssim b_n$ and $b_n \lesssim a_n$. 
If $\mu$ is a finite measure on some set $\cX$, we denote the $L^2(\cX , \mu)$-norm of a function 
$f: \cX \to \mbr$ by $||f||_{L^2}:=\left( \int_{\cX} |f(x)|^2 d\mu\right)^{1/2}$. For a finite set 
$\{x_1, ..., x_n\} \subset \cX$ we denote the empirical $L^2$-norm as 
 $||f||_n:= \left( \frac{1}{n}\sum_{j=1}^n |f(x_j)|^2  \right)^{\frac{1}{2}}$.


\section{Setup}
\label{sec:setting}

In this section we provide the mathematical framework for our analysis.

\subsection{Two-Layer Neural Networks}


We let $\cX \subset \mbr^d$ be the input space and $\cY\subseteq [-C_Y,C_Y]$, $C_Y > 0$, be the output space. 
The unknown data distribution on the data space $\cZ=\cX \times \cY$ is denoted by $\rho$ 
while the marginal distribution on $\cX$ is denoted as 
$\rho_X$ and the regular conditional distribution on $\cY$ given $x \in \cX$ is 
denoted by $\rho(\cdot | x)$, see e.g. \cite{Shao_2003_book}.

Given a measurable function $g: \cX \to \mbr$ we further define the expected risk as 
\begin{equation}
\label{eq:expected-risk}
\cE(g) := \mbe[ \ell (g(X), Y) ]\;,
\end{equation}  
where the expectation is taken w.r.t. the distribution $\rho$ and $\ell: \mbr \times \cY \to \mbr_+$ is the least-squares loss 
$\ell(t, y)=\frac{1}{2}(t-y)^2$. It is known that the global minimizer  of $\cE$ over the set of all measurable functions is 
given by the regression function $g_\rho(x)= \int_{\cY} y \rho(dy|x)$.

The hypothesis class considered in this paper is given by the following set of two-layer neural networks: 
Let $M \in \mbn$ be the 
network width, i.e. the number of hidden neurons. The network parameters are denoted 
by $a=(a_1, ..., a_M)^T \in \mbr^M$, the parameter of the 
input layer are denoted by $B=(b_1 ,..., b_M) \in \mbr^{d \times M}$ and $c=(c_1, ..., c_M)^T \in \mbr^M$ is the bias. 
We condense 
all parameters in $\theta = (a, B, c) \in \Theta$, with 
$\Theta = \mbr^M \times \mbr^{d \times M} \times \mbr^M$ being the parameter space with the euclidean 
vector norm 
\[
||\theta ||_\Theta^2 =   ||a||_2^2 + \sum_{m=1}^M ||b_m||_2^2 + ||c||_2^2 =  ||a||_2^2 + ||B||_F^2 + ||c||_2^2,   
\]
for any $\theta \in \Theta$.

Given an activation function $\sigma: \mbr \to \mbr$, we finally consider the class 
\begin{eqnarray*}
\cF_{M} & := \left\{ g_\theta :\cX \to \mbr \; : \; g_\theta(x) =  
\frac{1}{\sqrt M} \sum_{m=1}^M a_m \sigma( \inner{b_m , x} + \gamma c_m) \;, \right. \\
& \quad \left. \theta =(a, B, c) \in \mbr^M \times \mbr^{d \times M} \times \mbr^M \;, \gamma \in [0,1] \right\}\;. 
\end{eqnarray*}

The activation $\sigma$ is supposed to satisfy the following assumption:

\begin{assumption}[Activation Function]
\label{ass:neurons} 
the activation $\sigma: \mbr \to \mbr$ is two times differentiable with Lipschitz continuous second derivative 
and there exists a constant $C_\sigma <\infty$, 
such that 
 $\|\sigma^{\prime}\|_\infty \leq C_\sigma$, $\|\sigma^{\prime\prime}\|_\infty \leq  C_\sigma$.
\end{assumption}

Our goal is to minimize the expected risk \eqref{eq:expected-risk} over the set $\cF_M$, i.e. 
\begin{equation*}
\min_{g \in \cF_M} \cE(g)\;.
\end{equation*}
Here,  the distribution $\rho$ is known only through an i.i.d. sample $((x_1, y_1), ..., (x_n, y_n)) \in (\cX \times \cY)^n$. 
Hence, we seek for a solution for the empirical risk minimization problem 
\[  \min_{g \in \cF_M} \hat \cE(g) \;, \quad \hat \cE(g) = \frac{1}{n}\sum_{j=1}^n \ell(g(x_j, y_j)) \;. \] 

We aim at analyzing the generalization properties of gradient descent, whose basic iterations are given by   

\begin{align*}
\heavy_{t+1} &= \heavy_t - \alpha \nabla_\theta \hat \cE(g_{\theta_t})  \\
&= \heavy_t - \frac{\alpha}{n}\sum_{j=1}^n \ell'(g_{\theta_t}(x_j ) , y_j)  \nabla g_{\theta_t}(x_j)\;,  
\end{align*}
with $\alpha > 0$ being the stepsize and for some initialization $\theta_0 \in \Theta$.

{\bf Initialization.} 
Similar as in \cite{nitanda2020optimal} we assume a symmetric initialization. 
The parameters for the output layer are initialized as 
$a_{m}^{(0)}=\tau$ for $m \in\left\{1, \ldots, {M}/{2}\right\}$ and 
$a_{m}^{(0)}=-\tau$ for $m \in\{{M}/{2}+1, \ldots, M\}$, for some $\tau >0$. 
Let $\mu_{0}$ be a uniform distribution on the sphere 
$\mathbb{S}^{d-1}_1=\{b \in \mathbb{R}^{d} \mid\|b\|_{2}=1\} \subset \mathbb{R}^{d}$. 
The parameters for the input layer are initialized as $b_{m}^{(0)}=b_{m+{M}/{2}}^{(0)}$ 
for $m \in\{1, \ldots, {M}/{2}\}$, where $(b_{m}^{(0)})_{m=1}^{{M}/{2}}$
are independently drawn from the distribution $\mu_{0}$. 
The bias parameters are initialized as $c_{m}^{(0)}=0$ for 
$m \in\{1, \ldots, M\}$. The aim of the symmetric initialization is to make 
an initial function $g_{\theta_0}=0$, 
where $\theta_{0}=\left(a^{(0)}, B^{(0)}, c^{(0)}\right)$. 
Note that this symmetric trick does not have an impact on the limiting NTK, see \cite{zhang2020type},  
and is just for theoretical simplicity. Indeed, we can relax the 
symmetric initialization by considering an additional error stemming from 
the nonzero initialization in the function space.


\subsection{Neural Tangent Kernel}

The connection between kernel methods and neural networks is established via the NTK, 
see \cite{jacot2018neural, lee2019wide}. 
The gradient of $g_\theta$ w.r.t. $\theta$  at initialization $\theta_0 \in \Theta $ defines a feature 
map $\Phi_M: \cX \to \Theta$ by 
\[  \Phi_M(x) = \nabla_\theta g_\theta(x) \mid _{\theta=\heavy_0} \;. \]
This defines a kernel via 
\begin{align*}
 K_M(x, x') &= \inner{ \Phi_M(x) , \Phi_M(x') }_\Theta \;\\
&=\frac{1}{M} \sum_{r=1}^{M} \sigma\left(b_{r}^{(0) \top} x\right) \sigma\left(b_{r}^{(0) \top} x^{\prime}\right)+\frac{\left(x^{\top} x^{\prime}
+\gamma^{2}\right)}{M} \sum_{r=1}^{M} (a_r^{(0)})^2\sigma^{\prime}\left(b_{r}^{(0) \top} x\right) \sigma^{\prime}\left(b_{r}^{(0) \top} x^{\prime}\right) \;, 
\end{align*}
for any $x , x' \in \cX$. By \cite[Theorem 4.21]{steinwart2008support}, this kernel 
defines a unique RKHS $\cH_{M}$, given by 
\[ 
\cH_M =\{h:\mathcal{X}\rightarrow\mathbb{R}| \,\,\exists \,\theta\in\Theta \,\,\,s.t.\,\,\, 
h(x)= \left\langle\nabla g_{\theta_{0}}(x), \theta\right\rangle_{\Theta}\} \;,
\] 

Note that by the assumptions \ref{ass:neurons} and \ref{ass:input}, 
we see $K_M(x,x')\leq 4+2c_\sigma^2\tau^2 \eqqcolon \kappa^2$ 
for any $x, x' \in \operatorname{supp}(\rho_X)$.

We can consider $K_M$ as a random approximation of a kernel $K_\infty$, the neural tangent kernel (NTK). 
This is defined as a proper limit as $M\to \infty$:
\[ K_{\infty}\left(x, x^{\prime}\right) \coloneqq \mathbb{E}_{b^{(0)}}\left[\sigma\left(b^{(0) \top} x\right) \sigma\left(b^{(0) \top} x^{\prime}\right)\right]+\tau^{2}\left(x^{\top} x^{\prime}+\gamma^{2}\right) \mathbb{E}_{b^{(0)}}\left[\sigma^{\prime}\left(b^{(0) \top} x\right) \sigma^{\prime}\left(b^{(0) \top} x^{\prime}\right)\right],\]
see  \cite{jacot2018neural} for more information. 
Again, this kernel defines a unique RKHS $\cH_\infty$.


\section{Main Results}
\label{sec:main-results}

\subsection{Assumptions and Main Results}

In this section we formulate our assumptions and state our main results.

\begin{assumption}[Data Distribution]
\label{ass:input}
We assume that $|Y| \leq C_Y$ almost surely, for some $C_Y < \infty$. 
\end{assumption}

\vspace{0.2cm}

We let $\cL_\infty: L^2(\cX , \rho_X) \to L^2(\cX , \rho_X)$ denote the kernel integral operator 
associated to the NTK $K_\infty$. Note  that $\cL_\infty$ is bounded and  self-adjoint. Moreover, it is compact 
and thus has discrete spectrum $\{ \mu_j\}_j$, with $\mu_j \to 0$ as $j \to \infty$. By our assumptions, 
$\cL_\infty$ is of trace-class, i.e., has summable eigenvalues.


\vspace{0.2cm}

\begin{assumption}[Source Condition]
\label{ass:source}
Let $R>0$, $r\geq 0$. 
We assume 
\begin{align}
g_\rho  = \mathcal{L}_\infty^r h_\rho \;, \label{hsource}
\end{align}
for some $h \in L^2(\cX , \rho_X)$, satisfying $||h_\rho||_{L^2} \leq R$. 
\end{assumption} 

This assumption characterizes the hypothesis space and relates to the regularity of the regression function $g_\rho$. 
The bigger $r$ is, the smaller the hypothesis space is, the stronger the assumption is, and the easier the learning 
problem is, as 
$\cL_\infty^{r_{1}}\left(L^{2}\right) \subseteq \mathcal{L}_\infty^{r_{2}}\left(L^{2}\right)$ 
if $r_{1} \geq r_{2}$. Note that $g_\rho \in \cH_\infty$ holds for all $r\geq \frac{1}{2}$.

The next assumption relates to the capacity of the hypothesis space.

\vspace{0.2cm}

\begin{assumption}[Effective Dimension]
\label{ass:dim}
For any $\lambda >0$ we assume 
\begin{align}
\cN_{\cL_\infty}(\lam):= \operatorname{tr}\left(\cL_\infty(\cL_\infty+\lambda I)^{-1}\right) \leq c_{b} \lambda^{-b}, \label {effecDim}
\end{align}
for some $b \in[0,1]$ and $c_{b}>0$. 
\end{assumption}

The number $\cN_{\cL_\infty}(\lam)$ is called {\it effective dimension} or {\it degrees of 
freedom} \cite{Caponetto}. It is related to covering/entropy number conditions, 
see \cite{steinwart2008support}. The condition (\ref{effecDim}) is naturally satisfied with $b=1$, 
since $\cL_\infty$ is a trace class operator which implies that its eigenvalues $\left\{\mu_{i}\right\}_{i}$ 
satisfy $\mu_{i} \lesssim i^{-1}$. Moreover, if the eigenvalues of $\cL_\infty$ satisfy a polynomial 
decaying condition $\mu_{i} \sim i^{-c}$ for some $c>1$, or if $\cL_\infty$ is of finite rank, then the 
condition (\ref{effecDim}) holds with $b=1 / c$, or with $b = 0$. The case $b = 1$ is referred to as the 
capacity independent case. A smaller $b$ allows deriving faster convergence rates for the studied algorithms. 

\vspace{0.3cm }
{\bf Analysis of NTK Spectrum.}
It is known that a certain eigenvalue decay of the kernel integral operator $\cL_\infty$ implies 
a bound on the effective dimension. Thus, assumptions on the decay of the effective dimension directly 
relate to the approximation ability of the underlying RKHS, induced by the NTK. 
\\    
So far, only a few results are known that shed light on the RKHSs that are induced by specific NTKs and 
activation functions. The authors in \cite{bietti2019inductive} analyze the inductive bias of learning in the 
NTK regime by analyzing the NTK and the associated function space. They characterize the RKHS of the NTK for 
two-layer ReLU networks by providing a spectral decomposition of the kernel. 
This decomposition is based on a Mercer decomposition in the basis of spherical harmonics. 
Their analysis reveals a polynomial decay of eigenvalues, which leads to improved approximation 
properties compared to other function classes based on the ReLU activation. 
\\
The authors in \cite{geifman2020similarity}, \cite{chen2020deep} show that the NTK for fully connected networks with ReLU activation
is closely related to the standard Laplace kernel. For
normalized data on the hypersphere both kernels have the same eigenfunctions and
their eigenvalues decay polynomially at the same rate, implying that their RKHSs include the same sets of functions. 
Finally, \cite{bietti2020deep} show that for ReLU activations, the kernels derived from deep
fully-connected networks have essentially the same approximation properties as their shallow two-layer counterpart, 
namely the same eigenvalue decay for the corresponding integral operator.
\\
Little is known for other activations beyond ReLU. A notable exception is \cite{fan2020spectra} that 
studies the eigenvalue distributions of the finite-width Conjugate Kernel and of the finite-width NTK 
associated to multi-layer feedforward neural networks for twice differentiable activations. In an asymptotic
regime, where network width is increasing linearly in sample size, under random initialization of the weights, 
and for input samples satisfying a notion of approximate pairwise orthogonality, they show that the 
eigenvalue distributions of the CK and NTK converge to deterministic limits.


\vspace{0.4cm}

{\bf Rates of convergence.}
Our first general result establishes an upper bound for the excess risk in terms of the 
stopping time $T$ and the number of neurons $M$, under the assumption that the weights remain in a vicinity of the initialization. 
The proof is outlined in Section \ref{outline-of-proof} and further detailed in Section \ref{app:proof-error-bounds}.

\vspace{0.2cm}

\begin{theorem}
\label{theo}
Suppose Assumptions \ref{ass:neurons}, \ref{ass:input},  \ref{ass:source} and \ref{ass:dim} are satisfied. 
Assume further that $\alpha\in (0,\kappa^{-2})$. Let $(\eps_T)_{T \geq 2}$ be a decreasing sequence of positive real numbers. 
Assume that for all $M\geq \widetilde{M}_0(\delta , T)$, with probability at least $1-\delta$
\begin{equation}
\label{eq:ball}
   \forall \;\; t \in [T]\;:\; \;\; \| \theta_t -\theta_0 \|_\Theta \leq B_\tau (\delta ,T) \;.  
\end{equation}   
There exist an $M_0:= M_0(\delta , \eps_T, d) > 0$  
and $n_0(\delta) \in \mbn$, such that for all $n \geq n_0$ and 
$M \geq M_0 $, with probability at least $1-\delta$ we have
\begin{align}
\label{eq:final-bound}
\| g_{\theta_T}- g_\rho\|_{L_2(\rho_x)}
&\leq \frac{C_\sigma B^3_\tau (\delta ,T)}{\sqrt{M}}  + \eps_T + C\cdot \log^3(6/\delta )\; T^{-r} \;,
\end{align}
with $C < \infty $, $C_\sigma < \infty$ independent of $n, M, T$.
\end{theorem}

\vspace{0.2cm}

We immediately can derive the rates of convergence by balancing the terms on the right 
hand side in \eqref{eq:final-bound}. 

\vspace{0.2cm}

\begin{corollary}[Rate of Convergence]
\label{cor:rates}
Let the assumptions of Theorem \ref{theo} be satisfied and choose $\eps_T = T^{-r}$, $T_n = n^{\frac{1}{2r+b}}$ 
and $2r + b >1$.  
There exist an  $n_0 \in \mbn$, depending on $\delta, r, b$, such that for all $n \geq n_0$, 
with probability at least $1-\delta$ we have 
\begin{align*}
\| g_{\theta_{T_n}}- g_\rho\|_{L_2(\rho_x)}
&\leq  C\cdot \log^3(6/\delta )\; n^{-\frac{r}{2r+b}} \;,
\end{align*}
provided that 
\begin{align*}
M\geq d^{\frac{5}{2}} \tilde{C} \cdot  \log^6(T_n) \cdot 
\begin{cases}
   \log^{10}(96/\delta) \; T_n^{3-4r}    &: r\in (0,\frac{1}{2}) \\
    T_n^{2r}   &: r \in [\frac{1}{2},\infty) \;. 
\end{cases}
\end{align*} 
Here, the constants $C< \infty$, $\tilde C < 0$ depend on $\kappa, \alpha, r, b$, but not on $n$. 
\end{corollary}

\vspace{0.3cm}

Up to a logarithmic factor, the rate of convergence in Corollary \ref{cor:rates} is known to be minimax optimal 
in the RKHS framework \cite{Caponetto,Muecke2017op.rates}. Compared to \cite{nitanda2020optimal}, who 
establish rates of convergence in the same setting for SGD, we are able to circumvent the saturation observed there, i.e. our result holds 
for any $r>0$, satisfying the constraint $2r+b > 1$ (the {\it easy learning regime}). In contrast, 
the rates in  \cite{nitanda2020optimal} are optimal only in the case where $r \in [1/2, 1]$. 

Notably, the number of hidden neurons that are sufficient to establish this rate is comparable to the number of random features 
for learning in RKHSs, see e.g. \cite{rudi2017generalization, carratino2018learning, nguyen2023random}.


\vspace{0.4cm}

{\bf The weights barely move.}
Our next result shows that the Assumption \eqref{eq:ball} is indeed satisfied and the weights remain in a vicinity of the initialization 
$\theta_0$. The proof is provided in Appendix \ref{app:weight-bounds}.

\vspace{0.2cm}

\begin{theorem}[Bound for the Weights]
\label{theo:weights-not-moving}
Let $\delta \in (0,1]$ and $T \geq 3$. 
There exists an $\widetilde{ M}_0(\delta , T) \in \mbn$, defined in \eqref{eq:M-bounded}, 
such that for all $M \geq \widetilde{ M}_0(\delta , T)$, 
with $\rho^{\otimes n}$-probability at least $1-\delta$ it holds 
\[
 \forall \;\; t \in [T]\;:\; \;\; \| \theta_t -\theta_0 \|_\Theta \leq B_\tau\;,  
\] 
where 
\[ B_\tau := B_\tau(\delta , T) := 80\cdot \log(T)\cdot \cB_\delta(1/T)\;,  \]
with 
\[ \cB_\delta(\lam ) :=  \frac{3}{2} +  14\kappa\log\left( \frac{60}{\delta}\right)
\sqrt{\frac{ \mathcal{N}_{\mathcal{L}_{\infty}}(\lambda)\log(60/\delta) }{ \lambda n}}  \]
and for any $n \geq \tilde n_0$, given in \eqref{eq:n-final-weights}.
\end{theorem}

\vspace{0.3cm}

\begin{corollary}[Refined Bounds]
\label{cor:weights}
Suppose the assumptions of Theorem \ref{theo:weights-not-moving} are satisfied. Let $\lam_n = T_n^{-1}$, with 
$ T_n = n^{\frac{1}{2r +b}}$, $2r+b>1$ and set $\eps_T = T^{-r}$. 

\begin{enumerate}
\item Let $r \geq \frac{1}{2}$ and $n \geq n_0$, for some $n_0 \in \mbn$ depending on $\delta, r, b$. 
With probability at least $1-\delta$
\begin{equation}
\label{eq:ball-1}
 \sup_{t \in [T]} || \theta_ t  - \theta_0||_\Theta  
\leq 160 \cdot \log( T_n ) = 160 \cdot \log( n^{\frac{1}{2r +b}} ) \;. 
\end{equation}
The number of neurons required\footnote{We can improve the factor of $d^5$ at the expense of increasing the sample complexity. } is 
\[ M \geq C_{\kappa,\sigma, \alpha}\; d^5 \log^4(T_n) T_n^{2r}  \;. \]

\item Let $r \leq \frac{1}{2}$. With probability at least $1-\delta$ 
\begin{align*}
 B_\tau(\delta , T_n)
&\leq 1200\cdot \kappa \log^{3/2}(60/\delta)\; \log(T_n)\;  T_n^{1/2-r}\\
&=  1200\cdot \kappa \log^{3/2}(60/\delta)\; \log\left(n^{\frac{1}{2r +b}}\right) 
 \; n^{\frac{1-2r}{2(2r +b)}} \;.  
\end{align*} 
This holds if we choose 
\[ M \geq d^5\log^4(T_n)T_n^{3-4r} \;.\]
\end{enumerate}
\end{corollary}


\subsection{Outline of Proof} 
\label{outline-of-proof}

Our proof is based on a suitable error decomposition. To this end, 
we further introduce additional linearized iterates in $\cH_M$:

\begin{align}
\label{eq:def-recurions}
f_{t+1}^M &=  f_t^M - \frac{\alpha}{n}\sum_{j=1}^n \ell'(f_t^M(x_j ) , y_j)K_M(x_j , \cdot)   \;, \\ 
h_t &=  \left\langle\nabla g_{\theta_{0}}(x), \theta_{t}-\theta_{0}\right\rangle_{\Theta} \;,
\end{align}
with initialization $f^M_0=h_0=0$.

We may split 

\begin{align}
\label{errordecomp}
\|g_{\theta_T} - g_\rho\|_{L^2}
&\leq 
\|g_{\theta_T}  - \mathcal{S}_M h_T\|_{L^2} + 
\|\mathcal{S}_M (h_T - f_T^M )\|_{L^2}  + 
\|\mathcal{S}_Mf_T^M-  g_\rho\|_{L^2}  \;,
\end{align}
where $\mathcal{S}_M : \cH_M \hookrightarrow L^2(\cX , \rho_X)$ is the inclusion of $\cH_M$ into $L^2(\cX , \rho_X)$.

\vspace{0.3cm}

For the first error term in \eqref{errordecomp} we use a Taylor expansion in $\theta_{t}$ around the initialization 
$\theta_0$. For any $x \in \mathcal{X}$ and $t \in [T]$, we have
\begin{align}
\label{eq:taylor}
g_{\theta_{t}}(x)&= g_{\theta_{0}}(x) + \mathcal{S}_M\left\langle\nabla g_{\theta_{0}}(x), \theta_{t}-\theta_{0}\right\rangle_{\Theta} 
+ r_{(\theta_{t},\theta_{0})}(x) \nonumber \\
&= \mathcal{S}_M h_t(x)+r_{(\theta_{t},\theta_{0})}(x) \; .  
\end{align} 
Here, $r_{(\theta_{t},\theta_{0})}(x)$ denotes the  Taylor remainder and can 
be uniformly bounded by 
\[ \| r_{(\theta_{t},\theta_{0})} \|_\infty \lesssim 
B_\tau \; \frac{\|\theta_t-\theta_0\|_\Theta^2}{\sqrt{M}} \;, \] 
as Proposition \ref{prop6} shows. This requires the iterates $\{\theta_t\}_{t \in [T]}$ to stay close to the initialization 
$\theta_0$, i.e. 
\[ \sup_{t \in [T]} || \theta_t - \theta_0||_\Theta \leq B_\tau \;,\] 
with high probability, for some $B_\tau < \infty$. 
We show in Theorem \ref{theo:weights-not-moving} that this is satisfied for sufficiently many neurons.

\vspace{0.3cm}

The second error term in \eqref{errordecomp} can be made arbitrarily small, see Theorem \ref{prop:second-term}. 
More precisely, there exists a decreasing sequence $\{\eps_T\}_T$ of positive real numbers such that 
\[  \|\mathcal{S}_M (h_T - f_T^M )\|_{L^2}  \lesssim \eps_T \;,\] 
with high probability and for sufficiently many neurons, depending on $\eps_T$.

\vspace{0.3cm}

For the last error term in \eqref{errordecomp} we apply the results in \cite{nguyen2023random} and find 
that with high probability, 
\[ \|\mathcal{S}_Mf_T^M-  g_\rho\|_{L^2} \lesssim T^{-r} \;, \]
for sufficiently many neurons, see Proposition \ref{prop:random-feature-result}.

As a result, we arrive at an overall bound of Theorem \ref{theo}
\[ \| g_{\theta_T}- g_\rho\|_{L_2(\rho_x)}
\lesssim \frac{B^3_\tau (\delta ,T)}{\sqrt{M}}  + \eps_T +  T^{-r} \;. \]

\bibliographystyle{alpha}
\bibliography{bib_iteration}


\appendix

\section{Preliminaries}

For our analysis we need some further notation.  

We denote by $\mathcal{S}_M : \cH_M \hookrightarrow  L^2(\cX , \rho_X)$ the inclusion of $\cH_M$ into $L^2(\cX , \rho_X)$ for $M \in \mathbb{N}\cup \{\infty\}$.
The adjoint operator $\cS^{*}_M: L^{2}(\mathcal{X}, \rho_X) \longrightarrow \mathcal{H}_{M}$ is identified as
$$
\cS^{*}_M g=\int_{\mathcal{X}} g(x) K_{M,x} \rho_X(d x)\;,
$$
where $K_{M,x}$ denotes the element of $\mathcal{H}_{M}$ equal to the function $t \mapsto K_M(x, t)$. 
The covariance operator $\Sigma_M: \mathcal{H}_{M} \longrightarrow \mathcal{H}_{M}$ and the kernel 
integral operator $\mathcal{L}_M: L^2(\cX , \rho_X) \to L^2(\cX , \rho_X) $ are given by
\begin{align*}
   \Sigma_Mf&\coloneqq \cS^*_M\cS_M f = \int_{\mathcal{X}}\left\langle f, K_{M,x}\right\rangle_{\mathcal{H}_{M}} K_{M,x} \rho_X(d x) \;, \\ 
   \mathcal{L}_M f&\coloneqq \cS_M \cS^*_M f = \int_{\mathcal{X}} f(x) K_{M,x} \rho_X(d x)\;,
\end{align*}

which can be shown to be positive, self-adjoint, trace class (and hence is compact).
Here $K_{M,x}$ denotes the element of $\mathcal{H}_{M}$ equal to the function $x' \mapsto K_M(x, x')$. The
empirical versions of these operators are given by

\begin{center}
\begin{align*}
&\widehat{\cS}_{M}: \mathcal{H}_{M} \longrightarrow \mathbb{R}^{n},  &&\left(\widehat{\cS}_{M} f\right)_{j}=\left\langle f, K_{M,x_{j}}\right\rangle_{\mathcal{H}_{M}}, \\
&\widehat{\cS}_{M}^{*}: \mathbb{R}^{n} \longrightarrow \mathcal{H}_{M}, && \widehat{\cS}_{M}^{*} \mathbf{y}=\frac{1}{n} \sum_{j=1}^{n} y_{j} K_{M,x_{j}}, \\
&\widehat{\Sigma}_{M}:=\widehat{\cS}_{M}^{*} \widehat{\cS}_{M}: \mathcal{H}_{M} \longrightarrow \mathcal{H}_{M},&& \widehat{\Sigma}_{M}=\frac{1}{n} \sum_{j=1}^{n}\left\langle\cdot, K_{M,x_{j}}\right\rangle_{\mathcal{H}_{M}} K_{M,x_{j}}.
\end{align*}
\end{center}

We introduce the following definitions similar to definition 2 of  \cite{rudi2017generalization}.
This framework was originally established to bound the generalization error $\|f_t^M-g_\rho\|$ where $f_t^M$ follows the tikhonov algorithm with respect to $K_M$. 
Fortunately some results of \cite{rudi2017generalization} will also be useful to bound the weights of the neural network.
\begin{center}
\begin{align*}
&\mathcal{Z}_M: \mathbb{R}^{(d+1)M} \rightarrow L^2\left(X, \rho_X\right),  &&\left(\mathcal{Z}_M \theta\right)(\cdot)=\nabla g_{\theta_{0}}(x)^{\top} \theta, \\[5pt]
&\mathcal{Z}_M^*: L^2\left(X, \rho_X\right) \rightarrow \mathbb{R}^{(d+1)M},
&& \mathcal{Z}_M^* g=\int_X \nabla g_{\theta_{0}}(x) g(x) d \rho_X(x), \\[5pt]
&\widehat{\mathcal{Z}}_M: \mathbb{R}^{(d+1)M} \rightarrow \left(\mathbb{R}^n,n^{-\frac{1}{2}}\|.\|_2\right),  &&\left(\widehat{\mathcal{Z}}_M \theta \right)_i=\nabla g_{\theta_{0}}(x_i)^{\top} \theta, \\[5pt]
&\widehat{\mathcal{Z}}_M^*: \left(\mathbb{R}^n,n^{-\frac{1}{2}}\|.\|_2\right) \rightarrow \mathbb{R}^{(d+1)M},
&& \widehat{\mathcal{Z}}_M^* a=\frac{1}{n}\sum_{i=1}^n\nabla g_{\theta_{0}}(x_i)a_i , \\[5pt]
&\mathcal{C}_M: \mathbb{R}^{(d+1)M} \rightarrow \mathbb{R}^{(d+1)M},
&& \mathcal{C}_M=\int_X \nabla g_{\theta_{0}}(x) \nabla g_{\theta_{0}}(x)^{\top} d \rho_X(x),\\
&\widehat{\mathcal{C}}_M: \mathbb{R}^{(d+1)M} \rightarrow \mathbb{R}^{(d+1)M},
&& \widehat{\mathcal{C}}_M=\frac{1}{n} \sum_{i=1}^n \nabla g_{\theta_{0}}(x_i) \nabla g_{\theta_{0}}(x_i)^{\top}.
\end{align*}
\end{center}

\begin{remark}
Note that $\mathcal{C}_M$ and $\widehat{\mathcal{C}}_M$ are self-adjoint and positive operators, 
with spectrum is $[0, \kappa^2]$ and we further have 
$\mathcal{C}_M=\mathcal{Z}_M^*\mathcal{Z}_M$,   
$\widehat{\mathcal{C}}_M=\widehat{\mathcal{Z}}_M^*\widehat{\mathcal{Z}}_M$,  
$\mathcal{L}_M=\mathcal{Z}_M\mathcal{Z}_M^*$.  
\end{remark}


\section{Error Bounds}
\label{app:proof-error-bounds}

Recall the error decomposition from Section \ref{outline-of-proof}:

\begin{align}
\label{errordecomp2}
\|g_{\theta_T} - g_\rho\|_{L^2}
&\leq \underbrace{ \|g_{\theta_T} - g_{\theta_0} - \mathcal{S}_M h_T\|_{L^2}}_{I} + 
\underbrace{  \|\mathcal{S}_M (h_T - f_T^M )\|_{L^2}}_{II}  + 
\underbrace{  \|\mathcal{S}_Mf_T^M- \tilde g_\rho\|_{L^2}}_{III}  \;,
\end{align}
where we set $\tilde g_\rho =  g_\rho - g_{\theta_0}$.

This section is devoted to bounding each term on the rhs of \eqref{errordecomp2}. 
To this end, we need the following assumption:

\begin{assumption}
\label{ass:Taylor-is-satisfied}
Let $\delta \in (0,1]$ and $T \in \mbn$. There exists an $\widetilde{ M}_0(\delta , T) \in \mbn$, depending on $\delta$ and $T$, 
such that for all $M \geq \widetilde{ M}_0(\delta , T)$, with $\rho^{\otimes n}$-probability at least $1-\delta$ it holds 
\begin{equation}
\label{eq:AT}
 \forall \;\; t \in [T]\;:\; \;\; \| \theta_t -\theta_0 \|_\Theta \leq B_\tau\;,  
\end{equation} 
for some constant $B_\tau \geq 1+ \tau$.
\end{assumption}

We will show in Section \ref{app:weight-bounds} that this assumption is satisfied.


\subsection{Bounding I} 

In this section we provide an estimate of the first error term 
$ \|g_{\theta_T} - g_{\theta_0} - \mathcal{S}_M h_T\|_{L^2}$ in 
\eqref{errordecomp2}.

\vspace{0.3cm}

\begin{proposition}
\label{prop:taylor-remainder} 
Suppose Assumption \ref{ass:Taylor-is-satisfied} is satisfied. 
With probability at least $1-\delta$, we have 
for all $M \geq \widetilde{ M}_0(\delta , T)$ that
\[   \| g_{\theta_T} - g_{\theta_0} - \mathcal{S}_M h_T\|_{L^2} \leq    \frac{C_\sigma B_\tau^3}{\sqrt{M}}  \; , \]
for some $C_\sigma < \infty$ and for some $\widetilde{ M}_0(\delta , T) < \infty$. 
\end{proposition}

\vspace{0.2cm}

\begin{proof}[Proof of Proposition \ref{prop:taylor-remainder}]
From Assumption \ref{ass:Taylor-is-satisfied} and  Proposition \ref{prop6}, $a)$ we immediately obtain for all $x \in \cX$
\begin{align}
\label{eq:taylor2}
| g_{\theta_T}(x) - g_{\theta_0}(x) - \mathcal{S}_M h_T(x)|  &=  |r_{(\theta_0,\theta_T)}(x) | \nonumber \\
&\leq \frac{C_\sigma B_\tau}{\sqrt{M}} \; \|\theta_T-\theta_0\|^2_\Theta \nonumber \\
&\leq  \frac{C_\sigma B_\tau^3}{\sqrt{M}} \;, 
\end{align}
holding with $\rho^{\otimes n}$-probability at least $1-\delta$, for all $M\geq \widetilde{M}_0(\delta , T)$.  
\end{proof}


\subsection{Bounding II} 

In this section we estimate the second term $ \|\mathcal{S}_M (h_T - f_T^M )\|_{L^2}$ in \eqref{errordecomp2}. 
A short calculation proves the following recursion:

\vspace{0.3cm}

\begin{lemma}
\label{lem:general-recursion}
Let $t \in \mbn$, $M \in \mbn$. Define $\hat  u_t^M:=h_t - f_t^M $. Then $(\hat  u_t^M)_t$ follows the recursion $\hat  u_0^M = 0$ and 
\begin{align*}
 \hat u^M_{t+1} &= (Id - \alpha \widehat{\Sigma}_M) \hat  u_t^M - \alpha \hat  \xi_t^{(1)}  - \alpha \hat  \xi_t^{(2)} \\
&= \alpha \sum_{s=0}^{t-1} (Id - \alpha \widehat{\Sigma}_M)^s \left(\hat  \xi_{t-s}^{(1)} + \hat  \xi_{t-s}^{(2)} \right) \;, 
 \end{align*}
where
\begin{align*}
\hat  \xi_t^{(1)} &= \frac{1}{n} \sum_{j=1}^n(g_{\theta_t}(x_j)  -  y_j) \; 
\left\langle \nabla g_{\theta_0}, \nabla g_{\theta_t}(x_j) - \nabla g_{\theta_0}(x_j)  \right\rangle_{\Theta} \in \cH_M  \;, \\
\hat  \xi_t^{(2)} &=  \widehat{\cS}^*_M \bar r_{(\theta_0 , \theta_t)}  \in \cH_M \;,
\end{align*}
with  $\bar r_{(\theta_0 , \theta_t)} = (r_{(\theta_0 , \theta_t)}(x_1) , ..., r_{(\theta_0 , \theta_t)} (x_n)) $.
\end{lemma}


\vspace{0.4cm}

\begin{theorem}
\label{prop:second-term}
Let $\delta \in (0,1]$, $T \geq 2$ and $\alpha < 1/\kappa^2$. 
Suppose Assumptions \ref{ass:neurons}, \ref{ass:source}, \ref{ass:Taylor-is-satisfied} are satisfied. 
Let $(\eps_T)_{T \geq 2}$ be a decreasing sequence of positive real numbers. 
There exists an $M_0:= M_0(\delta , \eps_T, d) > 0$ defined in \eqref{eq:M-again}  
and $n_0:=n_0(\delta, d,r, T) $, defined in \eqref{eq:n-again}, such that for all $n \geq n_0$ and 
$M \geq M_0 $, with probability at least $1-\delta$
\begin{align*}
  \forall t \in [T]: \;\;\;  ||  \cS_M \hat u_{t}||_{L^2} &\leq \eps_T \;. 
\end{align*}
\end{theorem}

\begin{proof}[Proof of Theorem \ref{prop:second-term}]
Applying \cite[Proposition A.15]{nguyen2023random}  gives with probability at least $1-\delta/2$
\begin{align}
\label{eq:from-useful}
||  \cS_M \hat u_{T+1}||_{L^2} &= ||\Sigma_M^{\frac{1}{2}} \hat u_{t}||_{\cH_M} 
\leq 2 \; || \widehat{\Sigma}_M ^{\frac{1}{2}} \hat u_{T+1} ||_{\cH_M} + 2\sqrt{ \lambda} \; || \hat u_{T+1} ||_{\cH_M} \;, 
\end{align}
provided $M \geq M_1(d, \lam , \delta)$, $n \geq n_1(\lambda, \delta)$,
and where $M_1(d, \lam , \delta)$ is defined in \eqref{eq:the-one-and-only} and $ n_1(\lambda, \delta)$ is given in 
\eqref{eq:n1-def}.

Let $a \in \{0, 1/2\}$. Using Lemma \ref{lem:general-recursion}, we find 
\begin{align}
\label{eq:from-useful2}
|| \widehat{\Sigma}_M ^a \hat u_{T+1} ||_{\cH_M} 
      &=  \alpha \left\|  \sum_{s=0}^{T-1} \widehat{\Sigma}_M ^a (Id - \alpha \widehat{\Sigma}_M)^s \left(\hat  \xi_{T-s}^{(1)} + \hat  \xi_{T-s}^{(2)} \right)  \right\|_{\cH_M} \nonumber  \\
      &\leq \alpha^{1-a}  \sum_{s=0}^{T-1} || (\alpha\widehat{\Sigma}_M )^a (Id - \alpha \widehat{\Sigma}_M)^s \hat  \xi_{T-s}^{(1)}||_{\cH_M} + \nonumber \\
      &\hspace{0.2cm}      \alpha^{1-a}  \sum_{s=0}^{T-1} || (\alpha\widehat{\Sigma}_M )^a (Id - \alpha \widehat{\Sigma}_M)^s \hat  \xi_{T-s}^{(2)}||_{\cH_M} \;.
\end{align}

\vspace{0.4cm}


{\bf Bounding $\sum_{s=0}^{T-1} || (\alpha\widehat{\Sigma}_M )^a (Id - \alpha \widehat{\Sigma}_M)^s \hat  \xi_{T-s}^{(1)}||_{\cH_M}$:}
\\
\vspace{0.2cm}
To begin with, we further split the noise term $\xi_{t}^{(1)} \in \cH_M$ into $ \xi_{t}^{(1)} = \xi_{t}^{(11)}  +  \xi_{t}^{(12)}$, 
with 
\[ \xi_{t}^{(11)} = \frac{1}{n} \sum_{j=1}^n(g_{\theta_t}(x_j)  -  g_\rho(x_j)) \left\langle \nabla g_{\theta_0}, \nabla g_{\theta_t}(x_j) - \nabla g_{\theta_0}(x_j)  \right\rangle_{\Theta} \]
and 
\[  \xi_{t}^{(12)} =    \frac{1}{n} \sum_{j=1}^n(g_\rho(x_j)  -  y_j)\left\langle \nabla g_{\theta_0}, \nabla g_{\theta_t}(x_j) - \nabla g_{\theta_0}(x_j)  \right\rangle_{\Theta} \;. \]
Thus, 
\begin{align}
\label{eq:from-useful7}
\sum_{s=0}^{T-1} || (\alpha\widehat{\Sigma}_M )^a (Id - \alpha \widehat{\Sigma}_M)^s \hat  \xi_{T-s}^{(1)}||_{\cH_M}  &\leq 
\sum_{s=0}^{T-1} || (\alpha\widehat{\Sigma}_M )^a (Id - \alpha \widehat{\Sigma}_M)^s \hat  \xi_{T-s}^{(11)}||_{\cH_M}  \nonumber \\ 
&+ \sum_{s=0}^{T-1} || (\alpha\widehat{\Sigma}_M )^a (Id - \alpha \widehat{\Sigma}_M)^s \hat  \xi_{T-s}^{(12)}||_{\cH_M} \;. 
\end{align}
For the first term in \eqref{eq:from-useful7}, we apply Lemma \ref{lem:noise11} and obtain for all $t \in [T]$, for any  
$M \geq M_2 (r,d, T, \delta)$ defined in \eqref{eq:neurons}, with probability at least $1-\delta/8$
\begin{align*}
 || \xi_{t}^{(11)}||_{\cH_M} 
&\leq C_{\kappa, r, R, \alpha ,g_{\theta_0}} \; \log(96/\delta) \; \frac{C_{\sigma}B^2_\tau}{\sqrt M}  \;   
  \cdot \left(  \frac{C_{\sigma}B^2_\tau}{\sqrt M} + \right. \nonumber \\
& \left. + ||\Sigma_M^{\frac{1}{2}} \hat u_t||_{\cH_M} + \frac{ || \hat u_t||_{\cH_M}}{\sqrt{\alpha T}} + \frac{1}{(\alpha t)^{r}}  
 + \frac{1}{\sqrt n} + \frac{1}{(\alpha t)^{r/2} n^{\frac{1}{4}}}  \right) \\
&\leq 3 C_{\kappa, r, R, \alpha ,g_{\theta_0}} \; \log(96/\delta) \; \frac{C_{\sigma}B^2_\tau}{\sqrt M}  \;   
  \cdot \left(  \frac{C_{\sigma}B^2_\tau}{\sqrt M} +  \right. \nonumber \\ 
&  \left. + ||\Sigma_M^{\frac{1}{2}} \hat u_t||_{\cH_M} + \frac{ || \hat u_t||_{\cH_M}}{\sqrt{\alpha T}} + \frac{1}{(\alpha t)^r} \right) \;,
\end{align*}
for some $C_{\kappa, r, R, \alpha ,g_{\theta_0}} < \infty$ and if we let $ n \geq (\alpha T)^{2r}$. 

Moreover, by  Lemma \ref{prop0},  
\begin{align*}
|| (\alpha\widehat{\Sigma}_M )^a (Id - \alpha \widehat{\Sigma}_M)^s || 
&= \sup_{x \in [0,1]} (1-x)^s x^a 
\leq \left(\frac{a}{a+s}\right)^a \;,
\end{align*} 
where we use the convention $(0/0)^0 :=1$. Thus, 
\begin{align}
\label{eq:from-useful10}
& \sum_{s=0}^{T-1} || (\alpha\widehat{\Sigma}_M )^a (Id - \alpha \widehat{\Sigma}_M)^s \hat  \xi_{T-s}^{(11)}||_{\cH_M} \nonumber \\
 &\leq  C_{\kappa, r, R, \alpha ,g_{\theta_0}, \sigma} \; \log(96/\delta) \; \frac{B^2_\tau}{\sqrt M} \;
   \sum_{j=1}^3 S_j(a,T) \;,
\end{align}
where we define and estimate each summand $S_j(a,T)$, $j=1,2,3$, below. We make repeatedly use of Lemma \ref{prop1} 
and Lemma \ref{lem:calcs}:  

\begin{align*}
S_1(a,T) &:= \frac{B^2_\tau}{\sqrt M} \; \sum_{s=0}^{T-1}  \left(\frac{a}{a+s} \right)^a 
\leq \sqrt{2} \;B^2_\tau \; \frac{T^{1-a}}{\sqrt M} \;. 
\end{align*} 

\begin{align*}
S_2(a,T) &:= \sum_{s=0}^{T-1}  \left(\frac{a}{a+s} \right)^a \; 
\left( ||\Sigma_M^{\frac{1}{2}} \hat u_{T-s}||_{\cH_M} + \frac{ || \hat u_{T-s}||_{\cH_M}}{\sqrt{\alpha T}} \right) \;.
\end{align*}

\begin{align*}
S_3(a,T) &:= \frac{1}{\alpha^r} \; \sum_{s=0}^{T-1}  \left(\frac{a}{a+s} \right)^a \; \frac{1}{(T-s)^{r}} \\
&\leq \frac{8\cdot 2^{\max\{r, \frac{1}{2}\}}}{\alpha^r} \; \left( a\cdot T^{\frac{1}{2}-r} + \frac{\eta_r(T)}{T^a} \right) \;.
\end{align*} 

Plugging these estimates into \eqref{eq:from-useful10} gives 
\begin{align}
\label{eq:from-useful15}
& \sum_{s=0}^{T-1} || (\alpha\widehat{\Sigma}_M )^a (Id - \alpha \widehat{\Sigma}_M)^s \hat  \xi_{T-s}^{(11)}||_{\cH_M} \nonumber \\
 &\leq  C'_{\kappa, r, R, \alpha ,g_{\theta_0}, \sigma} \; \log(96/\delta) \; \frac{B^2_\tau}{\sqrt M} \;
  \left( \frac{T^{1-a}}{\sqrt M}\;B^2_\tau + a\cdot T^{\frac{1}{2}-r} + \frac{\eta_r(T)}{T^a}  \right. \nonumber  \\
& \left. \; + \sum_{s=0}^{T-1}  \left(\frac{a}{a+s} \right)^a \; 
\left( ||\Sigma_M^{\frac{1}{2}} \hat u_{T-s}||_{\cH_M} + \frac{ || \hat u_{T-s}||_{\cH_M}}{\sqrt{\alpha T}} \right)  \right) \;,
\end{align}
for some $C'_{\kappa, r, R, \alpha ,g_{\theta_0}, \sigma}  < \infty$. 

\vspace{1cm}

For the second term in \eqref{eq:from-useful7}, we apply Lemma \ref{lem:noise21} and have for all 
\begin{equation}
\label{eq:M-large-enough}
 M \geq \widetilde{M}_0(\delta/8 , T)  
\end{equation} 
with probability at least $1-\delta/8$ 
\[ || \hat \xi_{t}^{(12)} ||_{\cH_M} \leq C_{\sigma,  g_{\theta_0}} \; d^{\frac{5}{2}}\;\frac{\log(32/\delta)B_\tau^2}{\sqrt{n\cdot M}}\;, \]
for some $C_{\sigma,  g_{\theta_0}} < \infty$. 
Hence, with probability at least $1-\delta$
\begin{align}
\label{eq:from-useful9}
& \sum_{s=0}^{T-1} || (\alpha\widehat{\Sigma}_M )^a (Id - \alpha \widehat{\Sigma}_M)^s \hat  \xi_{T-s}^{(12)}||_{\cH_M} \nonumber \\
&\leq   C_{\sigma,  g_{\theta_0}} \; d^{\frac{5}{2}}\;\frac{\log(32/\delta)B_\tau^2}{\sqrt{n\cdot M}}\; \sum_{s=0}^{T-1}  \left(\frac{a}{a+s}\right)^a  \nonumber  \\
&\leq   C_{\sigma,  g_{\theta_0}} \;d^{\frac{5}{2}}\; \frac{\log(32/\delta)B_\tau^2}{\sqrt{n\cdot M}} \; T^{1-a}\;.
\end{align}
In the last step we  apply Lemma \ref{prop1} and find 
\[   \sum_{s=0}^{T-1}  \left(\frac{a}{a+s}\right)^a \leq \sqrt{2} \; T^{1-a} \;.  \]

Combining \eqref{eq:from-useful9}, \eqref{eq:from-useful15} with \eqref{eq:from-useful7} gives 
with probability at least $1-\delta/4$
\begin{align}
\label{eq:from-useful11}
& \sum_{s=0}^{T-1} || (\alpha\widehat{\Sigma}_M )^a (Id - \alpha \widehat{\Sigma}_M)^s \hat  \xi_{T-s}^{(1)}||_{\cH_M} \nonumber \\
&\leq  C_{\bullet} \; \log(96/\delta) \; \frac{B_\tau^2}{\sqrt M} \nonumber \\
& \; \; \times \; \left( F_a(r, n , M, T) +    \sum_{s=0}^{T-1}  \left(\frac{a}{a+s} \right)^a \; 
\left( ||\Sigma_M^{\frac{1}{2}} \hat u_{T-s}||_{\cH_M} + \frac{ || \hat u_{T-s}||_{\cH_M}}{\sqrt{\alpha T}} \right)  \right) \;, 
\end{align}
for some $C_{\bullet} < \infty$, depending on $\kappa, r, R, \alpha ,g_{\theta_0}, d, \sigma$ 
where we set 
\begin{align}
\label{eq:F}
F_a( r, n , M, T)  &:=  \frac{B_\tau^2 \; T^{1-a}}{\sqrt M}  + a\cdot T^{\frac{1}{2}-r} + \frac{\eta_r(T)}{T^a}  +  
  \frac{d^{\frac{5}{2}}B_\tau^2 T^{1-a}}{\sqrt n} \;,
\end{align}
provided \eqref{eq:M-large-enough} holds.

\vspace{0.4cm}


{\bf Bounding $\sum_{s=0}^{T-1} || (\alpha\widehat{\Sigma}_M )^a (Id - \alpha \widehat{\Sigma}_M)^s \hat  \xi_{T-s}^{(2)}||_{\cH_M}$:}
\\
\vspace{0.2cm}
By definition of $\xi_{t}^{(2)}$ we 
obtain\footnote{Recall that we define $||\bar y||_n := \frac{1}{\sqrt n} ||\bar y||_2$, for any $\bar y \in \mbr^n$. Moreover, 
$\widehat{\cS}^*_M: (\mbr^n , ||\cdot ||_n) \to (\cH_M , ||\cdot ||_{\cH_M})$. } 
\begin{align}
\label{eq:from-useful3}
\sum_{s=0}^{T-1} || (\alpha\widehat{\Sigma}_M )^a (Id - \alpha \widehat{\Sigma}_M)^s \hat  \xi_{T-s}^{(2)}||_{\cH_M} &= 
  \sum_{s=0}^{T-1} ||(\alpha\widehat{\Sigma}_M )^a (Id - \alpha \widehat{\Sigma}_M)^s  \widehat{\cS}^*_M \bar r_{(\theta_0 , \theta_{T-s})}||_{\cH_M} \nonumber \\
  &\leq \alpha^{-\frac{1}{2}} \sum_{s=0}^{T-1} ||(\alpha\widehat{\Sigma}_M )^{a} (Id - \alpha \widehat{\Sigma}_M)^s \widehat{\cS}^*_M|| \cdot 
||\bar r_{(\theta_0 , \theta_{T-s})}||_{n} \nonumber \\
&= \frac{\alpha^{-\frac{1}{2}}}{\sqrt{n}} \sum_{s=0}^{T-1} ||(\alpha\widehat{\Sigma}_M )^{a+1/2} (Id - \alpha \widehat{\Sigma}_M)^s || \cdot 
||\bar r_{(\theta_0 , \theta_{T-s})}||_{2} \;.
\end{align}
From \eqref{eq:taylor2}, for all $M\geq \widetilde{M}_0(\delta/4 , T)$, with probability at least $1-\delta/4$
\begin{align*}
||\bar r_{(\theta_0 , \theta_{T-s})}||^2_{2}&= \sum_{j=1}^n |r_{(\theta_0 , \theta_{T-s})}(x_j)|^2  \\
&\leq \sum_{j=1}^n\frac{C^2_{\sigma} B_\tau^9}{M} \;. 
\end{align*}
Hence, for all $s=0, ..., T-1$, 
\begin{equation}
\label{eq:from-useful4}
||\bar r_{(\theta_0 , \theta_{T-s})}||_{2} \leq \frac{\sqrt{n} \; C_{\sigma}B_\tau^3}{\sqrt M} \;.
\end{equation}
Moreover, applying Lemma \ref{prop0} yields 
\begin{align}
\label{eq:from-useful5}
||(\alpha\widehat{\Sigma}_M )^{a+\frac{1}{2}} (Id - \alpha \widehat{\Sigma}_M)^s || 
&= \sup_{x \in [0,1]} |(1-x)^s x^{a+\frac{1}{2}}  | \nonumber \\
&\leq \left(\frac{a+\frac{1}{2}}{a+\frac{1}{2}+s}\right)^{a+\frac{1}{2}} \;. 
\end{align}
Combining \eqref{eq:from-useful3}, \eqref{eq:from-useful4} and \eqref{eq:from-useful5} gives 
with Lemma \ref{prop1} with probability at least $1-\delta/4$
\begin{align}
\label{eq:from-useful6}
\sum_{s=0}^{T-1} || (\alpha\widehat{\Sigma}_M )^a (Id - \alpha \widehat{\Sigma}_M)^s \hat  \xi_{T-s}^{(2)}||_{\cH_M} 
&\leq \frac{C_{\sigma}B_\tau^3}{\alpha^{\frac{1}{2}} \sqrt M}
\sum_{s=0}^{T-1} \left(\frac{a+\frac{1}{2}}{a+\frac{1}{2}+s}\right)^{a+\frac{1}{2}} \nonumber  \\
&\leq \frac{4C_{\sigma}B_\tau^3}{\alpha^{\frac{1}{2}} \sqrt M} \; \log^{2a}(T)\cdot T^{\frac{1}{2}-a}\;.
\end{align}

\vspace{0.3cm}

{\bf Putting things together.} 
With \eqref{eq:from-useful6}, \eqref{eq:from-useful11} and \eqref{eq:from-useful2} we get with probability at least $1-\delta/2$
\begin{align}
\label{eq:from-useful16}
& || \widehat{\Sigma}_M ^a \hat u_{T+1} ||_{\cH_M}   \nonumber  \\
&\leq  \tilde C_{\kappa, r, R, \alpha ,g_{\theta_0}, \sigma} \; \log(96/\delta) \; \frac{B_\tau^2}{\sqrt M} \nonumber \\
& \; \; \times \; \left( \tilde  F_a(r, n , M, T) +    \sum_{s=0}^{T-1}  \left(\frac{a}{a+s} \right)^a \; 
\left( ||\Sigma_M^{\frac{1}{2}} \hat u_{T-s}||_{\cH_M} + \frac{ || \hat u_{T-s}||_{\cH_M}}{\sqrt{\alpha T}} \right)  \right) \;, 
\end{align}
for some $\tilde C_{\kappa, r, R, \alpha ,g_{\theta_0}} < \infty$ and where we set 
\begin{equation}
\label{eq:F2}
\tilde F_a( r, n , M, T)  :=  F_a( r, n , M, T) +  B_\tau \log^{2a}(T)\cdot T^{\frac{1}{2}-a} \;,
\end{equation}
where $F_a( r, n , M, T)$ is defined in \eqref{eq:F}.

\vspace{0.3cm}

Now recall that by \eqref{eq:from-useful} with $\lam = 1/(\alpha T)$ we have 
\begin{align*}
||\Sigma_M^{\frac{1}{2}} \hat u_{t}||_{\cH_M} 
&\leq 2 \;|| \widehat{\Sigma}_M ^{\frac{1}{2}} \hat u_{t} ||_{\cH_M} + 2 \;\frac{||\hat u_{t}||_{\cH_M}}{\sqrt{\alpha T}} \;.
\end{align*}  
Plugging this into \eqref{eq:from-useful16} gives 
\begin{align}
\label{eq:from-useful17}
& || \widehat{\Sigma}_M ^a \hat u_{T+1} ||_{\cH_M}   \nonumber  \\
&\leq  4\tilde C_{\kappa, r, R, \alpha ,g_{\theta_0}, \sigma} \; \log(96/\delta) \; \frac{B_\tau^2}{\sqrt M} \nonumber \\
& \; \; \times \; \left( \tilde  F_a( r, n , M, T) + \sum_{s=0}^{T-1} \left(\frac{a}{a+s} \right)^a \; 
\left( || \widehat{\Sigma}_M^{\frac{1}{2}} \hat u_{T-s}||_{\cH_M} 
+ \frac{|| \hat u_{T-s}||_{\cH_M}}{\sqrt{\alpha T}} \right)\right) \;, 
\end{align}
with $\tilde F_a( r, n , M, T)$ from \eqref{eq:F2}.

\vspace{0.3cm}

Setting now 
\begin{equation}
\label{eq:def-U}
  \hat U_t (T):= || \widehat{\Sigma}_M^{\frac{1}{2}} \hat u_{t}||_{\cH_M} + \frac{|| \hat u_{t}||_{\cH_M}}{\sqrt{\alpha T}} 
\end{equation}  
we find 
\begin{align}
\label{eq:the-iteration}
 \hat U_{T+1} (T) 
&\leq  C_{\bullet} \; \log(96/\delta) \; \frac{B^2_\tau}{\sqrt M}\nonumber \\
& \;\;\; \times \left( \tilde  F_{\frac{1}{2}}( r, n , M, T) + \frac{\tilde  F_{0}( r, n , M, T)}{ \sqrt{\alpha T}}  \right. \nonumber \\
& \left. \; +  \sum_{s=0}^{T-1} \left(\frac{1}{1+2s} \right)^{\frac{1}{2}} \;  \hat U_{T-s} (T)  +  
 \frac{1}{\sqrt{\alpha T}}\sum_{s=0}^{T-1}  \hat U_{T-s} (T)       \right) \;,
\end{align}
where $C_\bullet < \infty$ depends on $\kappa, r, R, \alpha, \sigma, g_{\theta_0}$.

We prove our claim now by induction over $T\in \mbn$. Having with probability at least $1-\delta$ 
\[ \forall \; t \in [T]\; : \;\; \hat U_{t} (T) \leq \frac{\eps_T}{2} \;, \] 
provided $M \geq M_0(\delta ,\eps_T, d)$, we obtain with \eqref{eq:the-iteration} 
\begin{align*}
\hat U_{T+1} (T) 
&\leq C_{\bullet} \; \log(96/\delta) \; \frac{B^2_\tau}{\sqrt M}\nonumber \\
& \;\;\; \times \left( \tilde  F_{\frac{1}{2}}( r, n , M, T) + \frac{\tilde  F_{0}( r, n , M, T)}{ \sqrt{\alpha T}}  
 +  \eps_T \; \sqrt{T}     \right) \;,
\end{align*}
Note that for some $C_\alpha < \infty$, with 
\[ n \geq T^{2r}\;, \quad M \geq B_\tau^4\; T^{2r}\;, \quad  \log(T) \leq \sqrt{T} \,, \]
we find 
\begin{align*}
 & \tilde  F_{\frac{1}{2}}( r, n , M, T) + \frac{\tilde  F_{0}( r, n , M, T)}{ \sqrt{\alpha T}} \\
& \leq C_{\alpha} \left( \frac{\eta_r(T)}{\sqrt T} +  d^{5/2}  T^{1/2-r} + B_\tau \sqrt{T} \right)\;.
\end{align*}
Thus, 
\begin{align*}
\hat U_{T+1} (T) 
&\leq C_{\bullet} \; \log(96/\delta) \; \frac{B^2_\tau}{\sqrt M} \; V_r(T, \eps_T)  \;,
\end{align*}
where we set 
\[  V_r(T, \eps_T) :=  \frac{\eta_r(T)}{\sqrt T} +  d^{5/2}  T^{1/2-r} + B_\tau \sqrt{T} +  \eps_T \; \sqrt{T}   \;. \]
Hence, 
\begin{equation}
\label{eq:UT}
 \hat U_{T+1} (T) \leq \frac{\eps_{T+1}}{2}\;, 
\end{equation} 
if we let 
\begin{equation}
\label{eq:anotherM}
 M \geq M_3(\delta, \eps_{T+1}):= 
\max\left\{ \; B_\tau^4 T^{2r} \; , \; C_\bullet \log^2(96/\delta) \; B_\tau^4 \; \frac{V_r(T, \eps_T )}{\eps_{T+1}} \; \right\} \;.
\end{equation} 

Combining the last bound with 
\eqref{eq:from-useful} proves the result for all 
\begin{equation}
\label{eq:M-again}
 M \geq M_0(\delta , \eps_{T+1}, d) := \max\{ \widetilde{M}_0(\delta , T+1) , 
M_1 (d, 1/(\alpha (T+1)) , \delta), M_2(r,d,T, \delta), M_3(\delta, \eps_{T+1}) \}
\end{equation} 
and 
\begin{equation}
\label{eq:n-again}
n \geq n_0 (\delta, d,r, T+1) := \max\{ n_1( 1/(\alpha (T+1)) , \delta), n_2(d, T, r)   \}\;.
\end{equation} 
\end{proof}


\subsubsection{Bounding the Noise}

\begin{proposition}
\label{prop:emp-ltwo}
Let $t \in [T]$, $r >0$, $\lam >0$, $\delta \in (0,1]$ and $\bar R = \max\{1, R\}$. 
Suppose that Assumptions \ref{ass:Taylor-is-satisfied}, \ref{ass:source} are satisfied. 
There exist an $n_0$, defined in \eqref{eq:assumpt-M-n}, and an $M_0$, defined in \eqref{eq:assumpt-M-n-2}, 
depending on $d, r, \lam , \alpha, t, \delta$, such that 
for all $n \geq n_0$, $M \geq M_0$, with probability at least $1-\delta$, we have 
\begin{align*}
 || g_{\theta_t} - g_\rho  ||_n  
 &\leq   \frac{C_{\sigma}}{ \sqrt M} B^2_\tau + 2\;||\Sigma_M^{\frac{1}{2}} \hat u_t||_{\cH_M} + 2\;\sqrt{\lam } \cdot || \hat u_t||_{\cH_M}  \\
& +  2\;C_r \bar R \; \log(12/\delta) \left( 1 + \sqrt{\lam \cdot \alpha t } \right) \; (\alpha t)^{-r} \\
& + \sqrt 2 \log^{\frac{1}{2}}\left( \frac{6}{\delta}\right) \cdot \left( \frac{B_{\alpha t}}{n} + \frac{V_{\alpha t}}{\sqrt n}  \right)^{\frac{1}{2}} + 
|| \cS_M  f^*_t -  ( g_\rho - g_{\theta_0})  ||_{L^2}  \;,
\end{align*}
for some $C_{\nabla g} < \infty$, $C_r < \infty$. 
Here, $B_{\alpha t}$, $V_{\alpha t}$ are defined in \eqref{eq:Balpha}, \eqref{eq:Valpha}, respectively.
\end{proposition}

\begin{proof}[Proof of Proposition \ref{prop:emp-ltwo}] 
Let
\[    f^*_t := \cS_M^* \phi_t(\cL_M) g_\rho  \in \cH_M \;, \] 
where $\phi_t$ denotes the spectral regularization function associated to gradient descent, see e.g. \cite{Muecke2017op.rates}.  
We decompose the error into 
\begin{align}
\label{eq:error-decomp-emp-ltwo}
|| g_{\theta_t} -  g_\rho  ||_n 
&\leq ||  g_{\theta_t} - (g_{\theta_0} + h_t)    ||_n + 
      || h_t + g_{\theta_0} - g_\rho  ||_n  \nonumber \\
&\leq     ||   r_{(\theta_0 , \theta_t)}   ||_n + || h_t -  f^*_t ||_n    +   ||  f^*_t - (g_\rho - g_{\theta_0}) ||_n\;.
\end{align}

{\bf Step I: Bounding $ ||   r_{(\theta_0 , \theta_t)}   ||_n$.} 
By Assumption \ref{ass:Taylor-is-satisfied} and applying Proposition \ref{prop6} gives for the remainder term with 
probability at least $1-\delta/3$, 
for all $t \in [T]$
\begin{align}
\label{eq:taylor-emp-ltwo}
 ||   r_{(\theta_0 , \theta_t)}   ||^2_n &= \frac{1}{n}\sum_{j=1}^n |r_{(\theta_0 , \theta_t)}  (x_j)|^2\nonumber \\
 &\leq \frac{1}{n}\sum_{j=1}^n \frac{C^2_{\sigma}B^2_\tau }{M} || \theta_t - \theta_0 ||^4_{\Theta} \nonumber \\
 &\leq  \frac{C^2_{\sigma }}{M} B^4_\tau \;. 
\end{align}

\vspace{0.3cm}

{\bf Step II: Bounding $ || h_t -  f^*_t ||_n$.} We decompose this error term into 
\begin{align}
\label{eq:term2-emp-ltwo}
|| h_t -  f^*_t ||_n 
&= ||  \widehat{\Sigma}_M^{\frac{1}{2}} ( h_t -  f^*_t) ||_{\cH_M} \nonumber \\
&\leq ||\widehat{\Sigma}_M^{\frac{1}{2}}\widehat{\Sigma}_{M,\lambda}^{-\frac{1}{2}}|| 
\cdot ||\widehat{\Sigma}_{M,\lambda}^{\frac{1}{2}}\Sigma_{M,\lambda}^{-\frac{1}{2}}|| 
\cdot ||\Sigma_{M,\lambda}^{\frac{1}{2}} (h_t -  f^*_t)||_{\cH_M} \nonumber \\
&\leq ||\widehat{\Sigma}_{M,\lambda}^{\frac{1}{2}}\Sigma_{M,\lambda}^{-\frac{1}{2}}|| 
\cdot ||\Sigma_{M,\lambda}^{\frac{1}{2}} (h_t -  f^*_t)||_{\cH_M} \;,
\end{align}
since 
$||\widehat{\Sigma}_M^{\frac{1}{2}}\widehat{\Sigma}_{M,\lambda}^{-\frac{1}{2}}|| \leq 1 $. 
From \cite[Proposition A. 15]{nguyen2023random}, 
for any $\lam>0$, $n\geq n_1( \lam , \delta)$, with 
\begin{equation}
\label{eq:n1-def}
n_1( \lam , \delta):=\frac{8\kappa^2 \tilde\beta (\lam , \delta)}{\lambda}
\end{equation}
\[ \tilde{\beta} (\lam , \delta):= 
\log\left( \frac{24 \kappa^2\left((\left(1+2\log\left(\frac{12}{\delta}\right)\right)4\mathcal{N}_{\mathcal{L}_{\infty}}(\lambda)+1\right)}{\delta\|\mathcal{L}_\infty\|} \right) \]
and $M \geq M_1(d, \lam , \delta)$, 
\begin{equation}
\label{eq:the-one-and-only}
  M_1(d, \lam , \delta) := \frac{8 (d+2)\kappa^2 \beta_\infty (\lam , \delta)}{\lambda}\vee 8\kappa^4\|\mathcal{L}_\infty\|^{-1}\log^2 \left(\frac{12}{\delta}\right) \;, 
\end{equation}  
with
\[ \beta_\infty (\lam , \delta) =\log \left(\frac{24 \kappa^2(\mathcal{N}_{\mathcal{L}_\infty}(\lambda)+1)}{\delta\|\mathcal{L}_\infty\|} \right)\;, \]
we have with probability at least $1-\delta /6$
\begin{equation}
\label{eq:donno}
 \left\|\widehat{\Sigma}_{M,\lambda}^{\frac{1}{2}}\Sigma_{M,\lambda}^{-\frac{1}{2}}\right\| \leq 2. 
\end{equation} 
Furthermore, 
\begin{equation}
\label{eq:htminusfstart}
 ||\Sigma_{M,\lambda}^{\frac{1}{2}} (h_t -  f^*_t)||_{\cH_M} \leq  
||\Sigma_{M}^{\frac{1}{2}} (h_t -  f^*_t)||_{\cH_M} + 
\sqrt{\lambda } \cdot || h_t -  f^*_t||_{\cH_M} \;.
\end{equation}

We proceed by writing for any $s \in [0, \frac{1}{2}]$
\begin{align}
\label{eq:s-dep}
||\Sigma_{M}^{s} (h_t -  f^*_t)||_{\cH_M} 
&\leq ||\Sigma_{M}^{s} (h_t -  f^M_t)||_{\cH_M}  + ||\Sigma_{M}^{s} (f^M_t -  f^*_t)||_{\cH_M} \nonumber \\
&=  ||\Sigma_{M}^{s} \hat u_t ||_{\cH_M} + ||\Sigma_{M}^{s} (f^M_t -  f^*_t)||_{\cH_M} \;,
\end{align}
where $f^M_t$ is defined in \eqref{eq:def-recurions} and $ \hat u_t$ is 
defined in Lemma \ref{lem:general-recursion}.

Proposition \cite[Proposition A.3]{nguyen2023random} 
shows that there exists an $n_2(d,r, \alpha, t, \delta)$ such that for 
any $n \geq n_2(\alpha, t, \delta)$, with 
probability at least $1-\delta/6$
\[ ||\Sigma_{M}^{s} (f^M_t -  f^*_t)||_{\cH_M} \leq C_r (\alpha t)^{-(r+s-\frac{1}{2})} \left( \log(12/\delta) + R \right) \;,\] 
for some $C_r < \infty$. For this bound to hold, we need the number of neurons $M$ to be sufficiently large, i.e. 
$M \geq M_2(d,r,\alpha, t, \delta)$, with 
\begin{align*}
M_2(d,r, \alpha, t, \delta) &:=  
\begin{cases}
 8 (d+2)\kappa^2 \cdot \alpha t \cdot \beta'_\infty \vee C_{\delta,\kappa} & r\in\left(0,\frac{1}{2}\right)\\
(8 (d+2)\kappa^2  \cdot \alpha t \cdot \beta'_\infty)\vee C_1^{\frac{1}{r}}\vee \frac{C_2}{(\alpha t)^{-(1+b(2r-1))} } 
\vee C_{\delta,\kappa} & r\in\left[\frac{1}{2},1\right] \\
\frac{C_3}{(\alpha t)^{-2r}}\vee C_{\delta,\kappa} & r \in(1,\infty),\\
\end{cases}
\end{align*}
where 
\[ C_{\delta,\kappa}= 8\kappa^4\|\cL_\infty\|^{-1}\log^2\left( \frac{12}{\delta}\right) \]
and $\beta'_\infty (\alpha t) $ is defined by 
\begin{equation}
\label{eq:beta-infty-pime}
  \beta'_\infty (\alpha t) := \log \left( \frac{24\kappa^2(\cN_{\cL_{\infty}}(1/(\alpha t))+1)}{\delta ||\cL_\infty||} \right) \;.
\end{equation}
Combining the last bound with \eqref{eq:s-dep} then gives with probability at least $1-\delta/6$
\begin{align*}
||\Sigma_{M}^{s} (h_t -  f^*_t)||_{\cH_M}  
&\leq ||\Sigma_M^{s} \hat u_t||_{\cH_M} +  C_r (\alpha t)^{-(r+s-\frac{1}{2})} \left( \log(12/\delta) + R \right) \\
&\leq  ||\Sigma_M^{s} \hat u_t||_{\cH_M} +  C_r \bar R\; (\alpha t)^{-(r+s-\frac{1}{2})}  \log(12/\delta)  \;,
\end{align*}
where we set $\bar R = \max\{R, 1\}$ and use that $1 \leq \log(12/ \delta)$, for all $\delta \in (0, 1]$.

Plugging this into \eqref{eq:htminusfstart} shows  
\begin{align}
\label{eq:htminusfstart2}
 ||\Sigma_{M,\lambda}^{\frac{1}{2}} (h_t -  f^*_t)||_{\cH_M} &\leq  
 ||\Sigma_M^{\frac{1}{2}} \hat u_t||_{\cH_M}   + 
 C_r \bar R\;  (\alpha t)^{-r} \; \log(12/\delta)  \nonumber  \\
 &+ \sqrt{\lam } \cdot || \hat u_t||_{\cH_M}   + 
 C_r \bar R\; \sqrt{\lam } \cdot(\alpha t)^{-(r-\frac{1}{2})} \;  \log(12/\delta)   \;,
\end{align}
with probability at least $1-\delta/6 $.

Finally, from \eqref{eq:htminusfstart2} and \eqref{eq:donno}, \eqref{eq:term2-emp-ltwo}, we get with 
probability at least $1-\delta/3$ 
\begin{align}
\label{eq:empht-final}
 || h_t -  f^*_t ||_n 
&\leq  2\; |\Sigma_M^{\frac{1}{2}} \hat u_t||_{\cH_M} + 2 C_r \bar R\;  (\alpha t)^{-r} \; \log(12/\delta)  \nonumber  \\
 &+ 2\; \sqrt{\lam } \cdot || \hat u_t||_{\cH_M} + 2C_r \bar R\; \sqrt{\lam } \cdot(\alpha t)^{-(r-\frac{1}{2})} \;\log(12/\delta) \;,
\end{align}
under the given assumptions.

\vspace{0.3cm}

{\bf Step III: Bounding $||  f^*_t - ( g_\rho - g_{\theta_0})  ||_n$.} The last term is decomposed as 
\begin{align*}
||  f^*_t -( g_\rho - g_{\theta_0})  ||_n &\leq \sqrt{||  f^*_t - ( g_\rho - g_{\theta_0})   ||^2_n - 
  ||  \cS_M f^*_t -( g_\rho - g_{\theta_0})  ||^2_{L^2} }
+ || \cS_M  f^*_t - ( g_\rho - g_{\theta_0})   ||_{L^2} \;.
\end{align*}
The first term in the above inequality is bounded by applying \\cite[Proposition A.23]{nguyen2023random}. 
Since by Lemma \ref{lem:norm-fstar}, for some $C'_{\kappa, R} < \infty$
\[ ||f^*_t ||_\infty \leq C'_{\kappa, R}\;  (\alpha t)^{\max\{0, \frac{1}{2}-r\}} \;, \]
we have for all $M \geq M_3(d, \alpha, t, \delta)$, 
\[ M_3 (d, \alpha, t, \delta) := 8(d+2) \kappa^2 \; \alpha t \; \beta'_\infty (\alpha t) \;,\] 
and $\beta'_\infty (\alpha t)$ from \eqref{eq:beta-infty-pime}, with probability at least $1-\delta /3$, 
\begin{align*}
 \left|  ||  f^*_t -  ( g_\rho - g_{\theta_0}) ||^2_n - ||  \cS_M f^*_t -  ( g_\rho - g_{\theta_0}) ||^2_{L^2} \right| 
&\leq 2\log\left( \frac{6}{\delta}\right) \cdot \left( \frac{B_{\alpha t}}{n} + \frac{V_{\alpha t}}{\sqrt n}  \right)\;,
\end{align*}
where for some $C_{\kappa, R} < \infty$ we set 
\begin{equation}
\label{eq:Balpha}
 B_{\alpha t} = 4\left( Q^2 + C^2_{\kappa, R}(\alpha t)^{-2\min\{0, r-\frac{1}{2}\}} \right)
\end{equation}
and
\begin{equation}
\label{eq:Valpha}
 V_{\alpha t} = \sqrt 2 \left( Q +  C_{\kappa, R}(\alpha t)^{-\min\{0, r-\frac{1}{2}\}}\right) \cdot 
 ||  \cS_M f^*_t -  ( g_\rho - g_{\theta_0})  ||_{L^2} \;. 
\end{equation}
Hence, 
\begin{equation}
\label{eq:lastterm2}
||  f^*_t -  ( g_\rho - g_{\theta_0}) ||_n \leq 
\sqrt 2 \log^{\frac{1}{2}}\left( \frac{6}{\delta}\right) \cdot \left( \frac{B_{\alpha t}}{n} + \frac{V_{\alpha t}}{\sqrt n}  \right)^{\frac{1}{2}} + 
|| \cS_M  f^*_t -  ( g_\rho - g_{\theta_0}) ||_{L^2} \;.
\end{equation}

\vspace{0.3cm}

{\bf Step IV: Combining all previous steps.}
Let 
\begin{align}
\label{eq:assumpt-M-n}
n_0 (d,r, \alpha, t, \lam, \delta ) &:= \max \{ n_1(\lam , \delta ) , n_2(d,r, \alpha, t, \delta) \} \;, \\
\label{eq:assumpt-M-n-2}
M_0 (d,r, \alpha, t, \lam , \delta ) &:= \max \{\widetilde{M}_0(\delta, T), M_1(d, \lam, \delta ) , M_2(d,r, \alpha, t, \delta) , M_3(d,\alpha, t, \delta)  \}   \;.
\end{align}
Combining now \eqref{eq:lastterm2} with \eqref{eq:empht-final}, \eqref{eq:taylor-emp-ltwo} and 
\eqref{eq:error-decomp-emp-ltwo} finally gives for all $n\geq n_0 (d,r, \alpha, t, \delta)$ 
and for all $M \geq M_0 (d,r, \alpha, t, \delta)$
\begin{align*}
 || g_{\theta_t}- g_\rho ||_n  
 &\leq  \frac{C_{\sigma}}{ \sqrt M} B^2_\tau + 2\;||\Sigma_M^{\frac{1}{2}} \hat u_t||_{\cH_M} + 2\;\sqrt{\lam } \cdot || \hat u_t||_{\cH_M}  \\
& +  2\;C_r \bar R \; \log(12/\delta) \left( (\alpha t)^{-r} + \sqrt{\lam } \cdot (\alpha t)^{-(r-\frac{1}{2})}  \right)  \\
& + \sqrt 2 \log^{\frac{1}{2}}\left( \frac{6}{\delta}\right) \cdot \left( \frac{B_{\alpha t}}{n} + \frac{V_{\alpha t}}{\sqrt n}  \right)^{\frac{1}{2}} + 
|| \cS_M  f^*_t -  ( g_\rho - g_{\theta_0})  ||_{L^2} \;,
\end{align*}
holding with probability at least $1-\delta$.
\end{proof}


\vspace{0.4cm}

\begin{lemma}
\label{lem:noise11}
Let Assumptions \ref{ass:input} and \ref{ass:Taylor-is-satisfied} be satisfied. For $t \in [T]$ define 
\[ \xi_{t}^{(11)} = \frac{1}{n} \sum_{j=1}^n(g_{\theta_t}(x_j)  -  g_\rho(x_j)) \left\langle \nabla g_{\theta_0}, \nabla g_{\theta_t}(x_j) - \nabla g_{\theta_0}(x_j)  \right\rangle_{\Theta} \]
There exists an $ M_0(r,d, \alpha, T, \delta) > 0$, defined in \eqref{eq:neurons2}, such that for all $M \geq M_0 (r,d, \alpha, T, \delta)$, with probability at least $1-\delta$, 
\begin{align*}
 || \xi_{t}^{(11)}||_{\cH_M} 
&\leq C_{\kappa, r, R, \alpha ,g_{\theta_0}} \; \log(12/\delta) \; \frac{C_{\sigma}B^2_\tau}{\sqrt M}  \;   
  \cdot \left(  \frac{C_{\sigma}B^2_\tau}{\sqrt M}  \right. \nonumber \\
& \left. + ||\Sigma_M^{\frac{1}{2}} \hat u_t||_{\cH_M} + \frac{ || \hat u_t||_{\cH_M}}{\sqrt{\alpha T}} + \frac{1}{(\alpha t)^{r}}  
 + \frac{1}{\sqrt n} + \frac{1}{(\alpha t)^{\frac{r}{2}} n^{\frac{1}{4}}}  \right) \;,
\end{align*}
for some $C_{\kappa, r, R, \alpha ,g_{\theta_0}, \sigma} < \infty$.
\end{lemma}

\begin{proof}[Proof of Lemma \ref{lem:noise11}]
By applying Proposition \ref{prop:emp-ltwo}, Proposition \ref{prop:LipschitzGradient} and H\"older's inequality we find 
under Assumption ref{ass:Taylor-is-satisfied} with probability at least $1-\delta$ 
\begin{align}
\label{eq:from-useful8}
 || \xi_{t}^{(11)}||_{\cH_M} 
&\leq || \nabla g_{\theta_0} ||_{\cH_M} C_{\sigma} \frac{ B^2_\tau }{\sqrt M} \; ||\theta_t - \theta_0||_{\Theta} \; \frac{1}{n}\sum_{j=1}^n|g_{\theta_t}(x_j)  -  g_\rho(x_j)| \nonumber \\ 
 &\leq || \nabla g_{\theta_0} ||_{\cH_M} C_{\sigma} \frac{ B^2_\tau }{\sqrt M} \;  ||g_{\theta_t} - g_\rho ||_n \nonumber \\
 &\leq || \nabla g_{\theta_0} ||_{\cH_M} C_{\sigma} \frac{ B^2_\tau }{\sqrt M}   \cdot \left(   \frac{C_{\sigma}B^2_\tau}{\sqrt M} 
  + 2\;||\Sigma_M^{\frac{1}{2}} \hat u_t||_{\cH_M} + 2\;\sqrt{\lam } \cdot || \hat u_t||_{\cH_M}\right. \nonumber \\
 & \left. +  2\;C_r \bar R \; \log(12/\delta) \left( 1 + \sqrt{\lam \cdot \alpha t } \right) \; (\alpha t)^{-r} \right. \nonumber \\
&  \left. + \sqrt 2 \log^{\frac{1}{2}}\left( \frac{6}{\delta}\right) \cdot \left( \frac{B_{\alpha t}}{n} + \frac{V_{\alpha t}}{\sqrt n}  \right)^{\frac{1}{2}} + 
|| \cS_M  f^*_t -  \bar g_\rho  ||_{L^2} \right) \;,
\end{align}
with $\bar g_\rho = g_\rho - g_{\theta_0}$ and provided $M \geq \widetilde{M}_0(\delta/2 , T)$. 
We now simplify the above term. To this end, choose $\lam = \frac{1}{\alpha T}$. Then 
\[ \lam \cdot \alpha t = \frac{\alpha t}{\alpha T} \leq 1 \;.\]
Furthermore, from \cite[Proposition A.1]{nguyen2023random} we have with $\lam = (\alpha t)^{-1}$
\begin{equation}
\label{eq:bound-RF}
 || \cS_M  f^*_t -  \bar g_\rho  ||_{L^2} \leq C_r\; R\; (\alpha t)^{-r} \;, 
\end{equation}
for some $C_r < \infty$. Note that for this bound to hold, we need $M \geq M_1(d, \alpha, T, \delta)$, where

\begin{align}
\label{eq:neurons}
 M_1(r, d, \alpha, T, \delta) &:=  
\begin{cases}
24 d\kappa^2 \; (\alpha T) \; \beta_\infty ( T, \delta) \vee C_{\delta,\kappa} & r\in\left(0,\frac{1}{2}\right)  \\
\frac{(24 d \kappa^2 \beta_\infty ( T, \delta) )\vee C_1 (T, \delta)^{\frac{1}{r}}}{(\alpha T)^{-1}}\vee (\alpha T)^{(1+b(2r-1))} C_2(T, \delta) 
\vee C_{\delta,\kappa} & r\in\left[\frac{1}{2},1\right]   \\
 C_{\kappa, r} \log^2(4/\delta) \; (\alpha T)^{2r} \vee C_{\delta,\kappa} & r \in(1,\infty) ,  \\
\end{cases}
\end{align}
for some $C_{\kappa, r} < \infty$ and where 
\[ C_1 (T, \delta) =2\left( 4\kappa\log \left( \frac{4}{\delta}\right)\right)^{2r-1}(24d\kappa^2\beta_\infty (T, \delta))^{1-r} \;, \]

\[  C_2 (T, \delta) =4 \left( 4c_b\kappa^2\log\left( \frac{4}{\delta} \right)\right)^{2r-1}(24 d\kappa^2\beta_\infty (T, \delta))^{2-2r} \;, \]

\[ \beta_{\infty} (T, \delta) := \log\left( \frac{8 \kappa^2 (1+ \cN_\infty (1/(\alpha T)))}{ \delta ||\cL_\infty ||} \right) \;, \]

\[ C_{\delta,\kappa} := 8\kappa^4\|\mathcal{L}_\infty\|^{-1}\log^2 \left( \frac{4}{\delta}\right) \;. \]
Recall the definition of $B_{\alpha t}$ and $V_{\alpha t}$ from \eqref{eq:Balpha}, \eqref{eq:Valpha}, respectively. Since 
$t^{-2\min\{0, r-\frac{1}{2}\}} \leq 1$ for any $t \in [T]$ we may bound 
\begin{equation}
\label{eq:Balphat-up}
 B_{\alpha t} \leq 4\left( Q^2 + C^2_{\kappa, R}(\alpha )^{-2\min\{0, r-\frac{1}{2}\}} \right) =: B_{\kappa, r, R, \alpha} \;.
\end{equation} 
Moreover, by the same reasoning and with \eqref{eq:bound-RF}, we find 
\begin{equation}
\label{eq:Valphat-up}
V_{\alpha t} 
\leq \sqrt{2}\left(  Q + C_{\kappa, R} \alpha^{-\min\{0, r-\frac{1}{2}\}}\right) \cdot || \cS_M f_t^* - \bar g_\rho||_{L^2}
= V_{\kappa, r, R, \alpha}  \cdot R\cdot (\alpha t)^{-r} \;, 
\end{equation} 
with 
\[  V_{\kappa, r, R, \alpha} = C_r \cdot \sqrt{2}\left(  Q + C_{\kappa, R} \alpha^{-\min\{0, r-\frac{1}{2}\}}\right)\;. \]
Hence, by \eqref{eq:from-useful8}, for all $t \in [T]$, with probability at least $1- \delta$, we obtain
\begin{align*}
 || \xi_{t}^{(11)}||_{\cH_M} 
&\leq C_{\kappa, r, R, \alpha ,g_{\theta_0}, \sigma} \; \log(12/\delta) \;  \frac{B^2_\tau}{\sqrt M}  \;   
  \cdot \left(  \frac{B^2_\tau}{\sqrt M}  \right. \nonumber \\
& \left. + ||\Sigma_M^{\frac{1}{2}} \hat u_t||_{\cH_M} + \frac{ || \hat u_t||_{\cH_M}}{\sqrt{\alpha T}} + \frac{1}{(\alpha t)^{r}}  
 + \frac{1}{\sqrt n} + \frac{1}{(\alpha t)^{\frac{r}{2}} n^{\frac{1}{4}}}  \right) \;,
\end{align*}
where we collected all constants in $C_{\kappa, r, R, \alpha , g_{\theta_0}, \sigma} <\infty$ and provided 
$M \geq M_0(r, d, \alpha, T , \delta)$, where 
\begin{equation}
\label{eq:neurons2}
M_0(r, d, \alpha, T , \delta) = \max\{ \widetilde{M}_0(\delta /2 , T) , M_1(r, d, \alpha, T , \delta)\;, \}
\end{equation}
with $ M_1(r, d, \alpha, T , \delta)$ from \eqref{eq:neurons}.
\end{proof}

\vspace{0.3cm}

\begin{lemma}
\label{lem:noise21}
Let Assumptions \ref{ass:input} and \ref{ass:Taylor-is-satisfied} be satisfied. For $t \in [T]$ define 
\[  \xi_{t}^{(12)} :=    \frac{1}{n} \sum_{j=1}^n(g_\rho(x_j)  -  y_j) 
         \left\langle \nabla g_{\theta_0}, \nabla g_{\theta_t}(x_j) - \nabla g_{\theta_0}(x_j)  \right\rangle_{\Theta} \in \cH_M \;. \]
Then, for all $M\geq \widetilde{M}_0(\delta , T)$, with probability at least $1-\delta$, 
\[ || \xi_{t}^{(12)} ||_{\cH_M}   \leq 64C_Y ||\nabla g_{\theta_0} ||_{\cH_M} \;C'_\sigma \;(d+1)^{5/2} \;\log(4/\delta) \;\frac{B^2_\tau}{\sqrt{n \cdot M}}  \;.\] 
\end{lemma}

\begin{proof}[Proof of Lemma \ref{lem:noise21}]
By Cauchy-Schwarz, we obtain 
\begin{align*}
|| \xi_{t}^{(12)} ||_{\cH_M} &\leq ||\nabla g_{\theta_0} ||_{\cH_M} 
\cdot \left\| \frac{1}{n} \sum_{j=1}^n(g_\rho(x_j)  -  y_j)  \cdot \left( \nabla g_{\theta_t}(x_j) - \nabla g_{\theta_0}(x_j) \right)\right\|_\Theta \;.
\end{align*}
Setting $z=(x, y) \in \cZ:=\cX \times \cY$ and  $g: \Theta \times \cZ \to \Theta$ with 
\[ g(\theta , z) := (g_\rho(x)  -  y)  \cdot \left( \nabla g_{\theta}(x) - \nabla g_{\theta_0}(x)  \right)  \;,\]
the map $g$ belongs to the function class 
\[  
\cG := \{ g: B_{B_\tau }(\theta_0 ) \times \cZ \to \Theta \;|\; 
g(\theta , z) = h(z)\cdot  \left( \nabla g_{\theta}(x) - \nabla g_{\theta_0}(x)  \right)  \} \;,
\]
defined in Proposition \ref{prop:Rademacher-Decomposition}, with $h(z) =g_\rho(x)  -  y $ and $||h||_\infty \leq 2C_Y$. 
Moreover, for any $g \in \cG$, we have 
\[ \mbe_{Y|X} [ g(\theta , z) ] = 0  \] 
and 
\[ G:= \sup_{g \in \cG} ||g||_\infty  \leq 2C_Y \; \frac{C_\sigma B^2_\tau}{\sqrt M} \;. \]
Applying Proposition \ref{prop:uniform-gen-bound} and Corollary \ref{cor:Rademacher-Gclass} therefore leads us to 
\begin{align*}
|| \xi_{t}^{(12)} ||_{\cH_M} 
&\leq  ||\nabla g_{\theta_0} ||_{\cH_M} \;  \left\| \frac{1}{n} \sum_{j=1}^n g(\theta_t , z_j) \right\|_\Theta \\
&\leq ||\nabla g_{\theta_0} ||_{\cH_M} \; \sup_{g \in \cG}\;  \left\| \frac{1}{n} \sum_{j=1}^n g(\theta , z_j) \right\|_\Theta \\
&\leq ||\nabla g_{\theta_0} ||_{\cH_M} \; \left(  16||h||_\infty c_\sigma \;   \frac{ B_\tau^2(d+1)^2}{\sqrt{M}}\;  \sqrt{\frac{d+1}{n}} 
     + G \; \sqrt{\frac{2\log(4/\delta)}{n}} + G \; \frac{4\log(4/\delta)}{n} \right) \\
&\leq ||\nabla g_{\theta_0} ||_{\cH_M} \; 
        \left(  32C_Y c_\sigma \; B^2_\tau (d+1)^{5/2}\;  \sqrt{\frac{1}{n\cdot M}}  +   
     4C_Y \; \frac{C_\sigma B^2_\tau}{\sqrt M} \; \frac{4\log(4/\delta)}{\sqrt n} \right) \\
&\leq 64C_Y ||\nabla g_{\theta_0} ||_{\cH_M} \;C'_\sigma B^2_\tau \; (d+1)^{5/2} \;\frac{\log(4/\delta)}{\sqrt{n \cdot M}}   \;,
\end{align*}
with probability at least $1-\delta$, provided $M \geq \widetilde{M}_0(\delta , T)$. 
\end{proof}

\vspace{0.3cm}

\begin{lemma}
\label{lem:norm-fstar}
Let
\[    f^*_t := \cS_M^* \phi_t(\cL_M) g_\rho  \in \cH_M \;, \] 
where $\phi_t$ denotes the spectral regularization function associated to gradient descent. Then 
\[ ||f^*_t ||_\infty \leq C'_{\kappa, R}\;  (\alpha t)^{\max\{0, \frac{1}{2}-r\}} \;, \]
for some $C'_{\kappa, R} < \infty$. 
\end{lemma}



\subsection{Bounding III} 

for bounding the last term in our error decomposition we use the results obtained in \cite{nguyen2023random}.

\begin{proposition}[Theorem 3.5 in \cite{nguyen2023random}]
\label{prop:random-feature-result}
Suppose Assumptions \ref{ass:neurons}, \ref{ass:source}, \ref{ass:input} and \ref{ass:dim} are satisfied. 
Let $\delta \in (0,1]$, $\alpha < 1/\kappa^2$. Let $T_n = n^{\frac{1}{2r+b}}$ and $2r + b >1$. 
With probability at least $1-\delta$ 
\[ \| \mathcal{S}_Mf_{T_n}^M-  g_\rho\|_{L^2} \leq C\; \log^3(2/\delta ) \; T_n^{-r } \;,\]
provided that 

\begin{align*}
M\geq \tilde{C}\log(n) \cdot \begin{cases}
T_n&: r\in\left(0,\frac{1}{2}\right)\\
T_n^{1+b(2r-1)}  &: r\in\left[\frac{1}{2},1\right] \\
T_n^{2r} &: r \in(1,\infty)\,\\
\end{cases}\\
\end{align*}
and $n\geq n_0:= e^{\frac{2r+b}{2r+b-1}}$. The constants $C< \infty$, $\tilde C< \infty$ 
do not depend on $n, \delta, M, T$.
\end{proposition}

\section{Proof: The weights barely move!}
\label{app:weight-bounds}


\begin{proof}[Proof of Theorem \ref{theo:weights-not-moving}] 
We prove this by induction over $T \in \mbn$.

Assume that the claim holds for $T \in \mbn$. We show that this implies the claim for $T+1$, for all $T \in \mbn$.

A short calculation shows that the weights $(\theta_t)_t$ follow the recursion 
\[ \theta_{T+1} - \theta_0 = \alpha \sum_{t=0}^{T}(Id - \alpha \widehat{\cC}_M)^t ( \zeta^{(1)}_{T-t} + 
 \zeta ^{(2)}_{T-t} + \zeta ^{(3)}_{T-t} + \zeta ^{(4)}_{T-t} + \zeta ^{(5)}_{T-t} )\;, \]
with
\begin{align*}
\zeta^{(1)}_{s} &= \frac{1}{n}\sum_{j=1}^n ( g_{\theta_s}(x_j) - g_{\rho}(x_j)) \cdot (\nabla g_{\theta_s} (x_j) - \nabla g_{\theta_0}(x_j) ) \;, \\
\zeta^{(2)}_{s} &= \frac{1}{n}\sum_{j=1}^n (g_{\rho}(x_j) - y_j) \cdot (\nabla g_{\theta_s} (x_j) - \nabla g_{\theta_0}(x_j) )\;, \\
\zeta^{(3)}_{s} &= \widehat{\cZ}^*_M \bar r_{(\theta_ 0 , \theta_{s}) } \\
\zeta^{(4)}_{s} &=  \widehat{\cZ}^*_M \by - \cZ_M^* g_\rho \\
\zeta^{(5)}_{s} &=  \cZ_M^* g_\rho \;. 
\end{align*}
Hence, 
\begin{align*}
||  \theta_{T+1} - \theta_0  ||_\Theta 
&\leq \alpha \sum_{t=0}^{T}\left\|  (Id - \alpha \widehat{\cC}_M)^s\zeta^{(1)}_{T-t} \right\|_\Theta   
  + \alpha  \sum_{t=0}^{T} \left\|  (Id - \alpha \widehat{\cC}_M)^s\zeta^{(2)}_{T-t} \right\|_\Theta   \nonumber \\
& \;\;\; + \alpha   \sum_{t=0}^{T} \left\|   (Id - \alpha \widehat{\cC}_M)^s\zeta^{(3)}_{T-t} \right\|_\Theta    
+   \alpha \sum_{t=0}^{T} \left\|   (Id - \alpha \widehat{\cC}_M)^s\zeta^{(4)}_{T-t} \right\|_\Theta \nonumber \\
& \;\;\; + \alpha \sum_{t=0}^{T} \left\|   (Id - \alpha \widehat{\cC}_M)^s\zeta^{(5)}_{T-t} \right\|_\Theta \;.
\end{align*}

\vspace{0.3cm}

{\bf Bounding $\alpha \sum_{s=0}^{t}\left\|  (Id - \alpha \widehat{\cC}_M)^s\zeta^{(1)}_{t-s} \right\|_\Theta  $.}

First, note that 
\begin{align*}
 \alpha \sum_{t=0}^{T}\left\|  (Id - \alpha \widehat{\cC}_M)^t\zeta^{(1)}_{T-t} \right\|_\Theta 
&\leq \alpha \sum_{t=0}^{T} \left\| (Id - \alpha \widehat{\cC}_M)^{T-t}  \right\| \cdot ||\zeta^{(1)}_{t}||_\Theta \;.
\end{align*}
We apply Proposition \ref{prop:emp-ltwo} with $\lam = 1/(\alpha (T+1))$ and proceed as in the proof of Lemma \ref{lem:noise11}.  
By additionally using Proposition \ref{prop:LipschitzGradient}, we therefore obtain  
with probability at least $1-\delta/10$
\begin{align*}
||\zeta^{(1)}_{t}||_\Theta 
&\leq C_\sigma \; \frac{ B^2_\tau }{n \cdot \sqrt{M} } \sum_{j=1}^n |  g_{\theta_t}(x_j) - g_{\rho}(x_j)| \nonumber \\
&\leq C_\sigma \; \frac{ B^2_\tau }{\sqrt{M}} \; || g_{\theta_t} - g_{\rho}||_n \nonumber \\
&\leq C_\sigma \; \frac{ B^2_\tau }{\sqrt{M}} \; \left( \frac{C_\sigma B^2_\tau}{\sqrt M}  + G_M(t, r, n, T+1)   \right)\;, 
\end{align*}
where for $t \in [T]$
\begin{equation}
\label{eq:G2}
G_M(t, r, n, T+1) := ||\widehat{\Sigma}_M^{\frac{1}{2}} \hat u_t||_{\cH_M} 
 + \frac{ || \hat u_t||_{\cH_M}}{\sqrt{\alpha (T+1)}} 
 + \frac{1}{(\alpha t)^{r}}  
 + \frac{1}{\sqrt n} + \frac{1}{(\alpha t)^{2r} n^{\frac{1}{4}}}  \;.
\end{equation}
Note that this holds for all $n \geq n_0 = n_0 (d, r , \alpha, t, T+1, \delta )$, and 
$M \geq M_0=M_0 (d, r , \alpha, t, T+1, \delta )$, 
where both, $n_0$ and $M_0$ depend on $d, r , \alpha, t, T+1, \delta $ and are 
defined in \eqref{eq:assumpt-M-n}, \eqref{eq:assumpt-M-n-2}, respectively.

We simplify the above bound. Observe that 
\[  \frac{1}{\sqrt n} \leq \frac{1}{(\alpha t)^{r}}\;, \]
if we assume that $n \geq (\alpha T)^{2r}$. 
Furthermore, under the same condition on the sample size, for any $t \in [T]$ and $r>0$, 
\[  \frac{1}{(\alpha t)^{r/2} n^{\frac{1}{4}}} \leq  \frac{1}{(\alpha t)^{2r}} \;. \]
By Theorem \ref{prop:second-term}, in particular \eqref{eq:UT}, 
with probability at least $1-\delta/10$, for all $t \in [T]$
\[ ||\widehat{\Sigma}_M^{\frac{1}{2}} \hat u_t||_{\cH_M} 
 + \frac{ || \hat u_t||_{\cH_M}}{\sqrt{\alpha (T+1)}}  \leq \frac{\eps_T}{2} \,, \]
if $n$,$M$ are sufficiently large, as given in Theorem \ref{prop:second-term}. 
Hence, 
\[ G_M(t, r, n, T+1) \leq \frac{\eps_T}{2} + \frac{3}{(\alpha t)^{r}} \;. \]
As a result, 
with probability at least $1-\delta/5$
\begin{align*}
 \alpha \sum_{t=0}^{T}\left\|  (Id - \alpha \widehat{\cC}_M)^t\zeta^{(1)}_{T-t} \right\|_\Theta 
&\leq \alpha  C_\sigma \; \frac{ B^2_\tau }{\sqrt{M}} \; \left( \frac{C_\sigma B^2_\tau \; T}{\sqrt M}  + 
  \frac{T \cdot \eps_T}{2} +  \sum_{t=0}^{T} \frac{3}{(\alpha t)^{r}}   \right) \nonumber \\
&\leq \alpha  C_\sigma \; \frac{ B^2_\tau }{\sqrt{M}} \; \left( \frac{C_\sigma B^2_\tau \; T}{\sqrt M}  + 
  \frac{T \cdot \eps_T}{2} +  \eta_r(T) \right)  \;.
\end{align*}
If we let $M \geq  M_1(\delta, T)$, with 
\begin{equation}
\label{eq:req-1}
 M_1 (\delta, T)  \geq 15\alpha C_\sigma^2 \; \max\left\{\; B_\tau^3 T \;, \;B_\tau T\eps_T \;,\; B_\tau \eta_r(T)  \; \right\}\;, 
\end{equation}
we finally obtain 
\begin{align}
\label{eq:the-first}
\alpha \sum_{t=0}^{T}\left\|  (Id - \alpha \widehat{\cC}_M)^t\zeta^{(1)}_{T-t} \right\|_\Theta 
&\leq \frac{B_\tau}{15} + \frac{B_\tau}{15} + \frac{B_\tau}{15} =\frac{B_\tau}{5} \;.
\end{align}

\vspace{0.3cm}

{\bf Bounding $\alpha \sum_{t=0}^{T}\left\|  (Id - \alpha \widehat{\cC}_M)^t\zeta^{(2)}_{T-t} \right\|_\Theta  $.} 

Following the lines of Lemma \ref{lem:noise21}, we obtain 
for all $M\geq \widetilde{M}_0(\delta , T)$, with probability at least $1-\delta/5$, 
\[ ||\zeta^{(2)}_{t}  ||_\Theta   \leq 64C_Y \;C'_\sigma B^2_\tau \; (d+1)^{5/2} \;\frac{\log(20/\delta)}{\sqrt{n \cdot M}}  \;.\] 
Hence, 
\begin{align}
\label{eq:the-second}
 \alpha \sum_{t=0}^{T}\left\|  (Id - \alpha \widehat{\cC}_M)^t\zeta^{(2)}_{T-t} \right\|_\Theta 
&\leq \alpha \sum_{t=0}^{T} \left\| (Id - \alpha \widehat{\cC}_M)^{T-t}  \right\| \cdot ||\zeta^{(2)}_{t}||_\Theta \nonumber \\
&\leq \tilde 
C_{\sigma} \; \frac{d^{5/2} B^2_\tau  \alpha T}{\sqrt{n \cdot M}} \nonumber \\
&\leq \frac{1}{5} B_\tau, 
\end{align}

provided that $M \geq M_2 (\delta, T)$, with  
\begin{equation}
\label{eq:req-2}
M_2 (\delta, T) = 25 \tilde C_\sigma^2 \; d^5\; B_\tau \;  \frac{(\alpha T)^2}{n^2} \;. 
\end{equation}

\vspace{0.3cm}

{\bf Bounding $\alpha \sum_{t=0}^{T}\left\|  (Id - \alpha \widehat{\cC}_M)^t\zeta^{(3)}_{T-t} \right\|_\Theta  $.}
We proceed as in the proof of Theorem \ref{prop:second-term}, Equ. \eqref{eq:from-useful4} and \eqref{eq:from-useful5}. 
Since with probability at least $1-\delta/5$, for all $t \in [T]$, 
\[ || \theta_t - \theta_0 ||_\Theta \leq B_\tau \;, \]
we may write 
\begin{align}
\label{eq:the-third}
\alpha \sum_{t=0}^{T}\left\|  (Id - \alpha \widehat{\cC}_M)^t\zeta^{(3)}_{T-t} \right\|_\Theta
&=  \alpha \sum_{t=0}^{T}\left\|  (Id - \alpha \widehat{\cC}_M)^{T-t}\zeta^{(3)}_{t} \right\|_\Theta \nonumber \\
&\leq  \frac{\sqrt{\alpha}}{\sqrt n} \sum_{t=0}^{T} \left\| (Id - \alpha \widehat{\cC}_M)^{T-t} (\sqrt{\alpha}\widehat{\cZ}_M^*)\right\| 
\cdot || \bar r _{(\theta_0 , \theta_s)}||_2 \nonumber \\ 
&\leq  \frac{ \sqrt{\alpha} C_\sigma B_\tau^3}{\sqrt M}  \;   \sum_{t=0}^{T} \left( \frac{1}{1+2(T-t)} \right)^{\frac{1}{2}}\nonumber \\ 
&\leq  C_\sigma B_\tau^3 \;  \sqrt{\frac{\alpha T}{M}} \nonumber \\
&\leq \frac{1}{5} \;  B_\tau \;,
\end{align}
provided that $M \geq M_3 ( \delta, T)$, with  
\begin{equation}
\label{eq:req-3}
M_3 (\delta, T) =  25 C_\sigma^2 B\tau^4 \cdot (\alpha T) \;. 
\end{equation}

\vspace{0.3cm}

{\bf Bounding $\alpha \sum_{s=0}^{t}\left\|  (Id - \alpha \widehat{\cC}_M)^s\zeta^{(4)}_{t-s} \right\|_\Theta  $.}

\begin{align*}
 \alpha \sum_{t=0}^{T}\left\|  (Id - \alpha \widehat{\cC}_M)^t\zeta^{(4)}_{T-t} \right\|_\Theta 
&= \alpha \sum_{t=0}^{T}\left\|  (Id - \alpha \widehat{\cC}_M)^t ( \widehat{\cZ}^*_M \by - \cZ_M^* g_\rho )   \right\|_\Theta  \\
&\leq  \alpha \sum_{t=0}^{T}\left\|  (Id - \alpha \widehat{\cC}_M)^t \widehat{\cC}_{M, \lam}^{1/2} \right\| 
\cdot || \widehat{\cC}_{M, \lam}^{-1/2} \cC_{M, \lam}^{1/2} || \nonumber \\ 
& \;\;\;\; \cdot  
\left\|\mathcal{C}_{M, \lambda}^{-1 / 2}\left(\widehat{\mathcal{Z}}_M^* y-\mathcal{Z}_M^* g_\rho\right) \right\|_\Theta \;.
\end{align*}

For the first term we obtain
\begin{align*}
 \alpha \sum_{t=0}^{T}\left\|  (Id - \alpha \widehat{\cC}_M)^t \widehat{\cC}_{M, \lam}^{1/2} \right\| 
&\leq \sqrt{ \alpha} \sum_{t=0}^{T}\left\|  (Id - \alpha \widehat{\cC}_M)^t (\alpha\widehat{\cC}_{M})^{1/2} \right\| + 
 \alpha  \sum_{t=0}^{T}\left\|  (Id - \alpha \widehat{\cC}_M)^t\right\| \\
&\leq 2\sqrt{2} \; \sqrt{ \alpha T}  + \sqrt{\lam } \cdot (\alpha T) \\
&\leq 4\sqrt{2} \; \sqrt{ \alpha T}\;,
\end{align*}
if we assume that $\lam \leq 1/(\alpha T)$.

From Proposition \ref{Opbound2}, with probability at least $1-\delta/10$
\[ || \widehat{\cC}_{M, \lam}^{-1/2} \cC_{M, \lam}^{1/2} || \leq 2 \;,\]
if 
\[ n\geq \frac{8\kappa^2 \tilde\beta(\lam)}{\lambda} \;, \]
with $\tilde{\beta} (\lam)$ given in \eqref{eq:beta-tilde-again}, if
\begin{equation}
\label{eq:req-44}
 M\geq M_4(\delta , T):=\frac{8 (d+2)\kappa^2 \beta_\infty (\lam )}{\lambda}\vee 8\kappa^4\|\mathcal{L}_\infty\|^{-1}\log^2 \left(\frac{20}{\delta} \right) 
\end{equation}
and where $\beta_\infty (\lam)$ is defined in \eqref{eq:beta-inf-3}. 

Applying Proposition \ref{prop:intermediate} gives with probability at least $1-\delta/10$
\[
\left\|\mathcal{C}_{M, \lambda}^{-1 / 2}\left(\widehat{\mathcal{Z}}_M^* y-\mathcal{Z}_M^* g_\rho\right)\right\|
\leq \sqrt{\lam} \; \cB_\delta(\lam )\;,
\]
where $ \cB_\delta(\lam )$ is defined in \eqref{eq:def-Bdelta} and if we let 
\begin{equation}
\label{eq:n-suff-large}
n \geq \frac{4 \kappa}{3\lam} \log(40/\delta) \;.
\end{equation}

Collecting all pieces gives with probability at least $1-\delta/5$
\begin{align}
\label{eq:the-fourth}
\alpha \sum_{t=0}^{T}\left\|  (Id - \alpha \widehat{\cC}_M)^t\zeta^{(4)}_{T-t} \right\|_\Theta 
&\leq 16\; \sqrt{\lam \cdot (\alpha T)}\; \cB_\delta(\lam ) \nonumber \\
&\leq 16\; \cB_\delta(\lam ) \;,
\end{align}
under the given assumptions.

\vspace{0.3cm}

{\bf Bounding $\alpha \sum_{s=0}^{t}\left\|  (Id - \alpha \widehat{\cC}_M)^s\zeta^{(5)}_{t-s} \right\|_\Theta  $.}

We have 
\begin{align*}
 \alpha \sum_{t=0}^{T}\left\|  (Id - \alpha \widehat{\cC}_M)^t\zeta^{(5)}_{T-t} \right\|_\Theta 
&= \alpha \sum_{t=0}^{T}\left\|  (Id - \alpha \widehat{\cC}_M)^t \cZ_M^* g_\rho   \right\|_\Theta  \nonumber \\
&\leq  \alpha \sum_{t=0}^{T}\left\|  (Id - \alpha \widehat{\cC}_M)^t \widehat{\cC}_{M, \lam}  \right\|
\cdot \left\|  \widehat{\cC}_{M, \lam}^{-1}\cC_{M, \lam}   \right\|  \nonumber \\
& \;\;
\cdot\left\|  \cC_{M, \lam} ^{-1} \cZ^*_M \cL_{M, \lam}^{\frac{1}{2}}  \right\|
\cdot \left\| \cL_{M, \lam}^{-\frac{1}{2}}\cL_{\infty, \lam}^{\frac{1}{2}} \right\|  
\cdot \left\|  \cL_{\infty, \lam}^{-\frac{1}{2}} \cL^r_{\infty} g_\rho \right\|_{L^2}  \;.
\end{align*}
For the first term we obtain 
\begin{align*}
\alpha \sum_{t=0}^{T}\left\|  (Id - \alpha \widehat{\cC}_M)^t \widehat{\cC}_{M, \lam}  \right\|_\Theta 
&\leq  \sum_{t=0}^{T}\left\|  (Id - \alpha \widehat{\cC}_M)^t (\alpha \widehat{\cC}_{M}) \right\|_\Theta + 
   \alpha \lam \cdot  \sum_{t=0}^{T}\left\|  (Id - \alpha \widehat{\cC}_M)^t\right\|_\Theta \\
&\leq 3\log(T) + \lam (\alpha T )\\
&\leq 4\log(T) \;,
\end{align*}
if we let $\lam \leq \log(T)/(\alpha T)$.

From Proposition \ref{Opbound4} we find probability at least $1-\delta/10$ 
\[  ||\widehat{\cC}_{M, \lam}^{-1}\cC_{M, \lam} || \leq \frac{3}{2} + \log\left( \frac{60}{\delta}\right) \; 
\sqrt{\frac{ 192 \kappa^2 \mathcal{N}_{\mathcal{L}_{\infty}}(\lambda)\log(60/\delta) }{ \lambda n}} \;,\]
provided that 
\[ n\geq \frac{8\kappa^2 (\log(60/\delta)+\tilde\beta (\lam))}{\lambda} \] 
with 
\begin{equation}
\label{eq:beta-tilde-again}
 \tilde{\beta} (\lam) := \log\left( \frac{40 \kappa^2(\left(1+2\log(20/\delta)\right)4\mathcal{N}_{\mathcal{L}_{\infty}}(\lambda)+1)}{\delta\|\mathcal{L}_\infty\|}\right) 
\end{equation}
and 
\[ M\geq \frac{8 (d+2)\kappa^2 \beta_\infty (\lam)}{\lambda}\vee \frac{8\kappa^4}{\|\cL_\infty \|} \log^2 \left(\frac{20}{\delta}\right) \]
with 
\begin{equation}
\label{eq:beta-inf-3}
 \beta_\infty (\lam) :=\log \left( \frac{40 \kappa^2(\mathcal{N}_{\mathcal{L}_\infty}(\lambda)+1)}{\delta\|\mathcal{L}_\infty\|} \right)\;.
\end{equation}

From Proposition \ref{Opbound1} we obtain almost surely 
\[ \|\mathcal{C}_{M,\lambda}^{-1}\mathcal{Z}^*_M\mathcal{L}_{M,\lambda}^{\frac{1}{2}}\|\leq 2 \;. \]

Furthermore, \cite[Proposition A.14]{nguyen2023random} gives with $1-\delta/10$ 
\[ \| \cL_{M, \lam}^{-\frac{1}{2}}\cL_{\infty, \lam}^{\frac{1}{2}}  \| \leq 2\;, \]
if 
\[ M \geq \frac{8 \kappa^2 (d+2) \beta_\infty(\lam )}{\lam}\;,  \] 
where $ \beta_\infty(\lam )$ is defined in \eqref{eq:beta-inf-3}.

By Assumption \ref{ass:source}, we easily find 
\[ \left\| \cL_{\infty, \lam}^{-\frac{1}{2}} \cL^r_{\infty} g_\rho \right\|_{L^2} \leq R \lam^{r - \frac{1}{2}} .\]

Combining the previous bounds gives with probability at least $1-\delta/5$
\begin{align}
\label{eq:the-fith}
 \alpha \sum_{t=0}^{T}\left\|  (Id - \alpha \widehat{\cC}_M)^t\zeta^{(5)}_{T-t} \right\|_\Theta 
&\leq 16\log(T) \cdot \cB_\delta(\lam ) \;,
\end{align}
where we set 
\begin{equation}
\label{eq:def-Bdelta}
 \cB_\delta(\lam ) :=  \frac{3}{2} +  14\kappa\log\left( \frac{60}{\delta}\right)
\sqrt{\frac{ \mathcal{N}_{\mathcal{L}_{\infty}}(\lambda)\log(60/\delta) }{ \lambda n}}  \;, 
\end{equation}
if 
\begin{equation}
\label{eq:req-55}
 M\geq M_5(\delta , T):= \max\left\{\; \frac{8 (d+2)\kappa^2 \beta_\infty (\lam)}{\lambda}\vee \frac{8\kappa^4}{\|\cL_\infty \|} \log^2 \left(\frac{20}{\delta}\right)     \; ,
    \;   \frac{8 \kappa^2 (d+2) \beta_\infty(\lam )}{\lam}\; \right\}
\end{equation}

\vspace{0.3cm}

{\bf Collecting everything.}
Finally, combining \eqref{eq:the-first}, \eqref{eq:the-second}, 
\eqref{eq:the-third}, \eqref{eq:the-fourth} and \eqref{eq:the-fith}  
gives with probability at least $1-\delta$,  for all $t \in [T+1]$
\begin{align*}
|| \theta_ t  - \theta_0||_\Theta 
&\leq \frac{3}{5} \;  B_\tau  \; + \; 16\cdot \cB_\delta(\lam ) \; + \;16 \cdot \log(T) \cdot \cB_\delta(\lam ) \\
&\leq \frac{3}{5} \;  B_\tau  \; + \; 32 \cdot \log(T) \cdot \cB_\delta(\lam ) \;,
\end{align*}
provided that $M \geq \widetilde{ M}_0(\delta , T)$, with 
\begin{equation}
\label{eq:M-bounded}
\widetilde{ M}_0(\delta , T)  := \max \left\{M_j (\delta ,T) \; : \; j =0, 1, ..., 5 \right\} 
\end{equation}
and 
\begin{equation}
\label{eq:n-final-weights}
n\geq \tilde n_0 := \max\left\{\; (\alpha (T+1))^{2r} \; , \; \frac{8\kappa^2 \tilde\beta(\lam)}{\lambda} \; , \; 
\frac{4 \kappa}{3\lam} \log(40/\delta) \;, \;  \frac{8\kappa^2 (\log(60/\delta)+\tilde\beta (\lam))}{\lambda} \;\right\} \;.
\end{equation}
Finally, setting 
\[ B_\tau := B_\tau(\delta , T) := 80\cdot \log(T)\cdot \cB_\delta(\lam )  \]
gives 
\[  32 \cdot \log(T) \cdot \cB_\delta(\lam ) = \frac{2}{5} \; B_\tau \]
and 
\[ \sup_{t \in [T]} || \theta_ t  - \theta_0||_\Theta  \leq B_\tau  \;.\]
\end{proof}


\vspace{0.4cm}

\begin{proof}[Proof of Corollary \ref{cor:weights}]
\begin{enumerate}
\item  Let $r \geq \frac{1}{2}$. A short calculation shows that $\cB_\delta(T_n) \leq 2$ if we let 
\[ n \geq ( 784\cdot \kappa^2 \log^3(60/\delta) )^{\frac{2r+b}{2r-1}} \; . \]
Hence, $B_\tau(\delta , T_n) = 160\log(T_n)$.

We finally bound the number of neurons that are required for  achieving this rate from Theorem \ref{theo:weights-not-moving}, \eqref{eq:M-bounded}. 
Some tedious calculations reveal that 
\begin{align*}
M_0(\delta , T) &= C_{\kappa, \alpha ,r} \log(24/\delta) \; d \; T_n^{2r}
 \max\left\{ 1 , \frac{\log(T_n^b)}{T_n^{2r-1}} \right\} \;, \\
M_1(\delta , T) &= C_{\kappa,\sigma, \alpha}\; T_n \log^3 (T_n) \;,  \\
M_2(\delta , T) &= C_{\kappa,\sigma, \alpha}\; d^5 \cdot \log(T_n)\; \frac{T_n^2}{n^2}  \;,\\
M_3(\delta , T) &= C_{\kappa,\sigma, \alpha}\; \log^4(T_n)\; T_n  \;,\\
M_4(\delta , T) &= C_{\kappa,\sigma, \alpha}\; \log(20/\delta T_n^b) T_n  \;,\\
M_5(\delta , T) &= C_{\kappa,\sigma, \alpha}\; \log(20/\delta T_n^b) T_n   \;. 
\end{align*}
Calculating the maximum gives 
\[ \widetilde{M}_0(\delta, T) = C_{\kappa,\sigma, \alpha}\; d \log^4(T_n) T_n^{2r} \; 
\max\left\{ 1 , \frac{\log(T_n^b)}{T_n^{2r-1}} \right\}  \;.\]
Here, we use that $T_n \leq T_n^{2r}$, $\log(T_n^b) \leq \log(T_n)$, 
\[ \log(20/\delta T_n^b) = \log(20/\delta) + \log(T_n) \leq 2 \log(T_n^b) \]
if we require that 
$n \geq (20/\delta)^{2+r/b}$. Moreover, 
\[ \max\left\{ 1 , \frac{\log(T_n^b)}{T_n^{2r-1}} \right\} = 1 \]
for $n$ sufficiently large.

\item If $r \leq \frac{1}{2}$, then 
\begin{align*}
 \cB_\delta (T_n) 
&\leq 15\cdot \kappa \log^{3/2}(60/\delta)\; T_n^{1/2-r}\\
&\leq 15\cdot \kappa \log^{3/2}(60/\delta)\; n^{\frac{1-2r}{2(2r +b)}} \;, 
\end{align*} 
if 
\[ n \geq \left( \frac{9}{4 \kappa^2 \log^3(60/\delta)} \right)^{\frac{2r+b}{2r-1}} \;. \]
Hence, 
\begin{align*}
 B_\tau(\delta , T_n)
&\leq 1200\cdot \kappa \log^{3/2}(60/\delta)\; \log(T_n)\;  T_n^{1/2-r}\\
&=  1200\cdot \kappa \log^{3/2}(60/\delta)\; \log\left(n^{\frac{1}{2r +b}}\right) 
 \; n^{\frac{1-2r}{2(2r +b)}} \;.  
\end{align*} 
A lower bound for the number of neurons follows the same way as in the other case.   
\end{enumerate}
\end{proof}

\section{Technical Inequalities}


\subsection{Lipschitz Bounds}

In this section we provide some technical inequalities. 
We start proving that the gradient of our neural network is pointewisely Lipschitz-continuous 
on a ball around $\theta_0$ which we define as 
$$ 
Q_{\theta_0}(R)\coloneqq \left\{\theta=(a,B)\in\Theta\subset\mathbb{R}^{M}\times\mathbb{R}^{(d+1)\times M}:\left\|\theta-\theta_0\right\|_{\Theta}\leq R\right\}
$$

\begin{proposition}
\label{prop:LipschitzGradient}
Suppose Assumption \ref{ass:neurons} is satisfied. 
We have for any $x \in \mathcal{X}$ and for any $\theta, \tilde{\theta} \in Q_{\theta_0}(R)$,
\begin{align*}
\sup _{x \in \mathcal{X}}\left\|\nabla g_{\theta}(x)-\nabla g_{\tilde{\theta}}(x)\right\|_{\Theta} 
\leq \frac{C_{\nabla g}(R)}{\sqrt{M}}\|\theta-\tilde{\theta}\|_{\Theta},
\end{align*}
where
\begin{align*}
C_{\nabla g}(R)&\coloneqq  c_\sigma \max\{ R, \tau+1\}\; , 
\end{align*}
for some $c_\sigma < \infty$
\end{proposition}

\begin{proof}[Proof of Proposition \ref{prop:LipschitzGradient}]
Recall the definition of our function class 
\begin{eqnarray*}
\cF_{M} & = \left\{ g_\theta :\cX \to \mbr \; : \; g_\theta(x) =  \frac{1}{\sqrt M} \sum_{m=1}^M a_m \sigma\left( \inner{b_m , \tilde{x}} \right) \;, \right. \\
& \quad \left. \theta =(a, B) \in \mbr^M \times \mbr^{d+1 \times M} \;, \gamma >0 \right\}\;, 
\end{eqnarray*}
with $\tilde{x}:=(\gamma,x) \in \mathbb{R}^{d+1}$.
We denote $b_{m}=\left(b_m^{1}, b_{m}^{2}, \ldots, b_{m}^{d+1}\right) \in \mathbb{R}^{d+1}$, $b_{m}^{j}=B_{j m}$. We consider the gradient of $g_{\theta}$ as an element in $\Theta$. The partial derivatives are given by
\[
\partial_{a_{m}} g_{\theta}(x)=\frac{1}{\sqrt{M}} \sigma\left(\left\langle b_{m}, \tilde{x}\right\rangle\right), \quad \partial_{b_{m}^{j}} g_{\theta}(x)=\frac{a_{m}}{\sqrt{M}} \sigma^{\prime}\left(\left\langle b_{m}, \tilde{x}\right\rangle\right) \tilde{x}_{j} .
\]
Denoting further $\nabla g_{\theta}(x)=\left(\partial_{a} g_{\theta}(x), \partial_{B} g_{\theta}(x)\right)$
\[
\partial_{a} g_{\theta}(x)=\left(\partial_{a_{1}} g_{\theta}(x), \ldots, \partial_{a_{M}} g_{\theta}(x)\right) \in \mathbb{R}^{M}, \quad \partial_{B} g_{\theta}(x)=\left(\partial_{b_{m}^{j}} g_{\theta}(x)\right)_{j, m} \in \mathbb{R}^{d \times M}
\]
we then have for any $\|\tilde{x}\|_2\leq 2$ 

\begin{align*}
\left\|\partial_{a} g_{\theta}(x)-\partial_{\tilde{a}} g_{\tilde{\theta}}(x)\right\|_{2}^{2} &=\sum_{m=1}^{M}\left(\partial_{a_{m}} g_{\theta}(x)-\partial_{\tilde{a}_{m}} g_{\tilde{\theta}}(x)\right)^{2} \\
&=\frac{1}{M} \sum_{m=1}^{M}\left(\sigma\left(\left\langle b_{m}, \tilde{x}\right\rangle\right)-\sigma\left(\left\langle\tilde{b}_{m}, \tilde{x}\right\rangle\right)\right)^{2} \\
& \leq \frac{\left\|\sigma^{\prime}\right\|_{\infty}^{2}}{M} \sum_{m=1}^{M}\left|\left\langle b_{m}-\tilde{b}_{m}, \tilde{x}\right\rangle\right|^{2} \\
& \leq \frac{4\left\|\sigma^{\prime}\right\|_{\infty}^{2}}{M} \sum_{m=1}^{M}\left\|b_{m}-\tilde{b}_{m}\right\|_{2}^{2} \\
&=\frac{4\left\|\sigma^{\prime}\right\|_{\infty}^{2}}{M}\|B-\tilde{B}\|_{F}^{2} .
\end{align*}

In addition, 
\[
\begin{aligned}
\left\|\partial_{B} g_{\theta}(x)-\partial_{\tilde{B}} g_{\tilde{\theta}}(x)\right\|_{F}^{2} &=\sum_{m=1}^{M} \sum_{j=1}^{d+1}\left(\partial_{b_{m}^{j}} g_{\theta}(x)-\partial_{\tilde{b}_{m}^{j}} g_{\tilde{\theta}}(x)\right)^{2} \\
&=\frac{1}{M} \sum_{m=1}^{M} \sum_{j=1}^{d+1}\left(a_{m} \sigma^{\prime}\left(\left\langle b_{m}, \tilde{x}\right\rangle\right) \tilde{x}_{j}-\tilde{a}_{m} \sigma^{\prime}\left(\left\langle\tilde{b}_{m}, \tilde{x}\right\rangle\right) \tilde{x}_{j}\right)^{2} \\
& \leq(I)+(I I),
\end{aligned}
\]
where
\[
(I)=\frac{2}{M} \sum_{m=1}^{M} \sum_{j=1}^{d+1}\left(a_{m}-\tilde{a}_{m}\right)^{2} \sigma^{\prime}\left(\left\langle b_{m}, \tilde{x}\right\rangle\right)^{2} \tilde{x}_{j}^{2} \leq \frac{4}{M}\left\|\sigma^{\prime}\right\|_{\infty}^{2}\|a-\tilde{a}\|_{2}^{2}
\]

and where
\[
\begin{aligned}
(I I)=& \frac{2}{M} \sum_{m=1}^{M} \sum_{j=1}^{d+1} \tilde{a}_{m}^{2}\left(\sigma^{\prime}\left(\left\langle b_{m}, \tilde{x}\right\rangle\right)-\sigma^{\prime}\left(\left\langle\tilde{b}_{m}, \tilde{x}\right\rangle\right)\right)^{2} \tilde{x}_{j}^{2} \\
\leq & \frac{2}{M} \sum_{m=1}^{M} \sum_{j=1}^{d+1}\left(\tilde{a}_{m}-a_{0, m}\right)^{2}\left(\sigma^{\prime}\left(\left\langle b_{m}, \tilde{x}\right\rangle\right)-\sigma^{\prime}\left(\left\langle\tilde{b}_{m}, \tilde{x}\right\rangle\right)\right)^{2} \tilde{x}_{j}^{2} \\
&+\frac{2}{M} \sum_{m=1}^{M} \sum_{j=1}^{d+1} a_{0, m}^{2}\left(\sigma^{\prime}\left(\left\langle b_{m}, \tilde{x}\right\rangle\right)-\sigma^{\prime}\left(\left\langle\tilde{b}_{m}, \tilde{x}\right\rangle\right)\right)^{2} \tilde{x}_{j}^{2} \\
\leq & \frac{8\left\|\sigma^{\prime\prime}\right\|_\infty^{2}}{M} \sum_{m=1}^{M} R^2\left\|b_{m}-\tilde{b}_{m}\right\|_{2}^{2}+a_{0, m}^{2}\left\|b_{m}-\tilde{b}_{m}\right\|_{2}^{2} \\
\leq & \frac{8\left\|\sigma^{\prime\prime}\right\|_\infty^{2}}{M}\left(R^2+\tau^{2}\right)\|B-\tilde{B}\|_{F}^{2}.
\end{aligned}
\]
Finally,
\[
\begin{aligned}
\left\|\nabla g_{\theta}(x)-\nabla g_{\tilde{\theta}}(x)\right\|_{\Theta}^{2} &=\left\|\partial_{a} g_{\theta}(x)-\partial_{\tilde{a}} g_{\tilde{\theta}}(x)\right\|_{2}^{2}+\left\|\partial_{B} g_{\theta}(x)-\partial_{\tilde{B}} g_{\tilde{\theta}}(x)\right\|_{F}^{2} \\
& \leq \frac{\left\|\sigma^{\prime}\right\|_{\infty}^{2}}{M}\|B-\tilde{B}\|_{F}^{2}+\frac{4}{M}\left\|\sigma^{\prime}\right\|_{\infty}^{2}\|a-\tilde{a}\|_{2}^{2}+\frac{8\left\|\sigma^{\prime\prime}\right\|_\infty^{2}}{M}\left(R^2+\tau^{2}\right)\|B-\tilde{B}\|_{F}^{2}\\
& \leq \frac{8}{M} \max \left\{\left\|\sigma^{\prime}\right\|_{\infty}^{2},\left\|\sigma^{\prime\prime}\right\|_\infty^{2} \left(R^2+\tau^2 \right)\right\}\left(\|a-\tilde{a}\|_{2}^{2}+\|B-\tilde{B}\|_{F}^{2}\right)\\
&\leq \frac{C_{\nabla g}^2}{M}\|\theta-\tilde{\theta}\|_{\Theta}^2.
\end{aligned}
\]
\end{proof}

Since $(\mathbb{R}^{M+(d+1)M},\|.\|_2)$ is isometric isomorphic  to 
$(\mathbb{R}^M \times \mathbb{R}^{(d+1) \times M} ,\|.\|_\Theta)$, for the next proposition we 
consider $\theta = (a,\vec{B}) \in\mathbb{R}^{M+(d+1)M}$ as a vector with  $\vec{B}=(b_1^T, \dots, b_M^T)^T$. 
Therefore we can define the inner product $\left \langle. , .\right\rangle_\Theta$ as the standard euclidean vector inner product.


\vspace{0.3cm}

\begin{proposition}
\label{prop6}
Denote by $r_{(\theta,\bar{\theta})}$ the remainder term of the Taylor expansion from
\begin{align*}
g_{\theta}-g_{\bar{\theta}} &= \left\langle\nabla g_{\bar{\theta}}, \theta-\bar{\theta}\right\rangle_{\Theta}+r_{(\theta,\bar{\theta})}.
\end{align*}
Suppose Assumption \ref{ass:neurons} holds and $\theta, \bar{\theta} \in Q_{\theta_0}(R)$. 
Then, for all $x\in\cX$, we have 
\begin{align*}
&a)\,\,\,\,\,\,\, |r_{(\theta,\bar{\theta})}(x)| \leq \frac{C_{\nabla g}(R)}{\sqrt{M}} \|\theta-\bar{\theta}\|^2_\Theta,\\[10pt]
&b) \,\,\,\,\,\,\, \left\langle\nabla g_{\theta}(x)-\nabla g_{\bar{\theta}}(x), \theta-\bar{\theta}\right\rangle_{\Theta}\leq \frac{C_{\nabla g}(R)}{\sqrt{M}}\|\theta-\bar{\theta}\|^2_{\Theta}.
\end{align*}
where $C_{\nabla g}(R)$  is defined in Proposition \ref{prop:LipschitzGradient}. 
\end{proposition}

\begin{proof}[Proof of Proposition \ref{prop6}]
$a)\,$ Recall that the remainder term $r_{(\theta,\bar{\theta})}$ of the Taylor expansion is given by
\[
r_{(\theta,\bar{\theta})}(x) =\left(\theta-\bar{\theta}\right)^T \nabla^2 g_{\tilde{\theta}}(x)\left(\theta-\bar{\theta}\right).
\]
where $\nabla^2 g_{\tilde{\theta}}(x)$ denotes the Hessian matrix evaluated at some $\tilde{\theta}$ on the line between $\theta$ and $\bar{\theta}$. 
Note that the remainder term can be bounded by the Lipschitz constant $L$ of the gradients i.e. for any $\mathbf{v} \in \Theta$ we have,
\begin{align*}
    \begin{aligned}
 \nabla^2 g_{\tilde{\theta}}(x) \mathbf{v} &= \lim _{h \rightarrow 0} \frac{\nabla g_{\tilde{\theta}+h\mathbf{v}}(x)-\nabla g_{\tilde{\theta}}(x)}{h} \\
 &\leq \lim _{h \rightarrow 0} \frac{\|\nabla g_{\tilde{\theta}+h\mathbf{v}}(x)-\nabla g_{\tilde{\theta}}(x)\|_2}{|h|} \\
 &\leq \lim _{h \rightarrow 0} L \frac{|h|\|\mathbf{v}\|_2}{|h|} \\
 &\leq L\|\mathbf{v}\|.
\end{aligned}
\end{align*}
Proposition \ref{prop:LipschitzGradient} shows that for the set $Q_{\theta_0}(R)$ this Lipschitz constant is given by $L = C_{\nabla g}/\sqrt{M}$.
Therefore by setting $\mathbf{v}= \theta-\bar{\theta}$ we obtain for the remainder term,
\[
|r_{(\theta,\bar{\theta})}(x)| \leq \frac{C_{\nabla g}(R)}{\sqrt{M}} \|\theta-\bar{\theta}\|^2_\Theta.
\]
$a)\,$ Using again Proposition \ref{prop:LipschitzGradient} together with Cauchy - Schwarz inequality, we obtain for the second inequality,
\[\left\langle\nabla g_{\theta}(x)-\nabla g_{\bar{\theta}}(x), \theta-\bar{\theta}\right\rangle_{\Theta}
\leq \frac{C_{\nabla g}(R)}{\sqrt{M}}\|\theta-\bar{\theta}\|^2_{\Theta}.\]
\end{proof}


\subsection{Uniform Bounds in Hilbert Spaces}

In this section we recall uniform convergence bounds for functions with values in a real separable Hilbert space. Our material is taken from 
\cite{stankewitz2023inexact} (Appendix B), \cite{foster2018uniform} (Lemma 4) and \cite{maurer2016vector}.

\begin{definition}[Rademacher Complexities]
Let $(\cH, ||\cdot ||)$ be a real separable Hilbert space. Further, let $\cG$ be a class of maps $g: \cZ \to \cH$ and $\bZ=(Z_1, ..., Z_n) \in \cZ^n$ be a vector 
of i.i.d. random variables. We define the \emph{empirical Rademacher complexity} and the \emph{population Rademacher complexity} of $\cG$ by
\[  \widehat{\Re}_n(\cG ) := \mbe_\eps \left[  \sup_{g \in \cG}\left\| \frac{1}{n}\sum_{j=1}^n \eps_j \; g(Z_j)  \right\| \right] \;, 
\quad \Re_n(\cG):= \mbe_Z[\widehat{\Re}_n(\cG )] \;,\]
respectively, where $\eps = (\eps_1, ..., \eps_n) \in \{-1, +1\}^n$ is a vector of i.i.d. Rademacher random variables, independent of $\cZ$.   
\end{definition}

\begin{proposition}[Uniform Bound]
\label{prop:uniform-gen-bound}
Let $(\cH, ||\cdot ||)$ be a real separable Hilbert space. Further, let $\cG$ be a class of maps $g: \cZ \to \cH$ 
and $\bZ=(Z_1, ..., Z_n) \in \cZ^n$ be a vector of i.i.d. random variables. Assume that 
\[ G:= \sup_{g \in \cG} ||g||_\infty < \infty \;. \]
For any $\delta \in (0,1]$, with probability at least $1-\delta$ we have 
\begin{align*}
\sup_{g \in \cG} \left \|  \mbe\left[ g(Z) \right] - \frac{1}{n}\sum_{j=1}^n g(Z_j) \right\| &\leq 
4 \widehat{\Re}_n(\cG ) + G \; \sqrt{\frac{2\log(4/\delta)}{n}} + G \; \frac{4\log(4/\delta)}{n} \;.
\end{align*} 
\end{proposition}

\begin{proof}[Proof of Proposition \ref{prop:uniform-gen-bound}]
Applying McDiarmid's bounded difference inequality, see e.g. Corollary 2.21 in \cite{wainwright2019high}, 
we obtain that on an event with probability at least $1-\delta/2$
\begin{align*}
\sup_{g \in \cG} \left \|  \mbe\left[ g(Z) \right] - \frac{1}{n}\sum_{j=1}^n g(Z_j) \right\| &\leq 
\mbe\left[  \sup_{g \in \cG} \left \|  \mbe\left[ g(Z) \right] - \frac{1}{n}\sum_{j=1}^n g(Z_j) \right\|  \right] 
 + G\; \sqrt{\frac{2\log(4/\delta)}{n}}\;.
\end{align*} 
Applying Lemma 4 from \cite{foster2018uniform} gives with probability at least $1-\delta/2$
\[ \mbe\left[  \sup_{g \in \cG} \left \|  \mbe\left[ g(Z) \right] - \frac{1}{n}\sum_{j=1}^n g(Z_j) \right\|  \right]  \leq  
4 \widehat{\Re}_n(\cG ) + 4G \;  \frac{\log(4/\delta)}{n} \;.\]
Combining both gives the result.
\end{proof}

\vspace{0.3cm}

\begin{proposition}
\label{prop:Rademacher-Decomposition} 
Let $R >0$, $R_1>0$, $R_2>0$, $\cZ = \cX \times \cY$ and $h: \cZ \to \mbr$ be some measurable function. 
Define  
\[ \cG:= \left\{ g: Q_{\theta_0}(R) \times \cZ \to \Theta \;|\; 
                 g(\theta ,z) = h(z)\cdot (\nabla g_\theta (x) - \nabla g_{\theta_0}(x)) \right\} \;. \]
For $j,k \in [d]$, define further the classes 
\[ \cG_1^{(j)} :=  \left\{ g: B_d(R_1) \times \cZ \to \mbr \; | \; g(b, z) = h(z) \sigma'(\langle b , x \rangle)x^{(j)} \right\}\;,  \]
\[ \cG_2^{(j, k)}:= \left\{ g: B_d(R_2) \times \cZ \to \mbr \; | \; 
       g(b, z) = h(z) \sigma^{''}(\langle b , x \rangle ) x^{(j)} x^{(k)} \right\} \;, \]
where $B_d(R_j)$ denotes the closed ball of radius $R_j$ in $\mbr^d$.    
The empirical Rademacher complexities satisfy the relation 
\[  \widehat{\Re}_n(\cG ) \leq  \frac{R}{\sqrt{M}}\; \sum_{j=1}^d \widehat{\Re}_n(\cG_1^{(j)} ) + 
       \frac{\sqrt{2}R}{\sqrt{M}}\;  \sum_{j,k=1}^d\widehat{\Re}_n(\cG_2^{(j, k)} )  \;. \]              
\end{proposition}

\begin{proof}[Proof of Proposition \ref{prop:uniform-gen-bound}]
 
\end{proof}

\vspace{0.3cm}

We next give a bound for the empirical Rademacher complexity of a function class comprised of real-valued Lipschitz functions. 
Although we believe that the proof is standard, 
we include it here for completeness sake.

\begin{proposition}
\label{prop:Rademacher-Lipschitz}
Let $\bZ=(Z_1, ..., Z_n) \in \cZ^n$ be a vector 
of i.i.d. random variables and $C>0$. 
Define the function class  
\[ \mathcal{G}_C=\{f_a(z):\cZ \rightarrow \mathbb{R}: a\in \mathbb{R}^p, 
 \|a\|_2\leq C \text{ such that } |f_{a_1}(z)-f_{a_2}(z)|\leq L \|a_1-a_2\|_2\} \;.\]  
Then
\[ \widehat{\Re}_n(\mathcal{G}_C ) \leq L \cdot C\sqrt{\frac{p}{n}} \;.  \]
\end{proposition}

\begin{proof}[Proof of Proposition \ref{prop:Rademacher-Lipschitz}]
Set $f(a,z):=f_a(z)$. By definition,
$$
\frac{1}{n} \mathbb{E}_r\left[\sup _{ \|a\|_2\leq C} \sum_{i=1}^n f\left(a,z_i\right)r_i\right]=\frac{1}{n} {\mathbb{E}}_{r_1, \ldots, r_{n-1}}\left[\mathbb{E}_{r_n}\left[\sup _{ \|a\|_2\leq C} u_{n-1}(a)+f\left(a,z_n\right)r_n\right]\right],
$$
where $u_{n-1}(a)=\sum_{i=1}^{n-1} f\left(a,z_i\right)r_i$. By definition of the supremum, for any $\epsilon>0$, there exist $\|a_{n,1}\|_2, \|a_{n,2}\|_2\leq C $ such that

\begin{align*}
& u_{n-1}\left(a_{n,1}\right)+f(a_{n,1},z_n) \geq(1-\epsilon)\left[\sup _{ \|a\|_2\leq C} u_{n-1}(a)+f\left(a,z_n\right)\right] \\
\text { and } & u_{n-1}\left(a_{n,2}\right)-f\left(a_{n,2},z_n\right) \geq(1-\epsilon)\left[\sup _{ \|a\|_2\leq C} u_{n-1}(a)-f\left(a,z_n\right)\right] .
\end{align*}

Thus by definition of $\mathbb{E}_{r_n}$,
\begin{align*}
& (1-\epsilon) \mathbb{E}_{r_n}\left[\sup _{ \|a\|_2\leq C} u_{n-1}(a)+r_nf(a,z_n)\right] \\
= & (1-\epsilon)\left[\frac{1}{2}\sup _{ \|a\|_2\leq C}\left[u_{n-1}(a)+f(a,z_n)\right]+\frac{1}{2}\sup _{a \in A} \left[u_{n-1}(a)-f(a,z_n)\right]\right] \\
\leq & \frac{1}{2}\left[u_{n-1}\left(a_{n,1}\right)+f(a_{n,1},z_n)\right]+\frac{1}{2}\left[u_{n-1}\left(a_{n,2}\right)-f(a_{n,2},z_n)\right] .
\end{align*}

Now set  $s_{j}^{n} =\operatorname{sgn}\left(a_{n,1}^{(j)}-a_{n,2}^{(j)}\right)$ for all $j\in[p]$ (note that $s^n_{j,\epsilon}$ is independent of $r_n$). Then, the previous inequality implies together with the Lipschitz property:

\begin{align*}
& (1-\epsilon) \mathbb{E}_{r_n}\left[\sup _{ \|a\|_2\leq C} u_{n-1}(a)+r_nf(a,z_n)\right] \\
& \leq \frac{1}{2}\left[u_{n-1}\left(a_{n,1}\right)+u_{n-1}\left(a_{n,2}\right)+L \sqrt{\sum_{j=1}^p  \left(a_{n,1}^{(j)}-a_{n,2}^{(j)}\right)^2}  \right] \\
& \leq \frac{1}{2}\left[u_{n-1}\left(a_{n,1}\right)+u_{n-1}\left(a_{n,2}\right)+L \sum_{j=1}^p s_{j}^{n} \left(a_{n,1}^{(j)}-a_{n,2}^{(j)}\right)  \right] \\
& =\frac{1}{2}\left[u_{n-1}(a_{n,1})+L\sum_{j=1}^ps_{j}^{n} a_{n,1}^{(j)}\right]+\frac{1}{2}\left[u_{n-1}\left(a_{n,2}\right)-L\sum_{j=1}^p s_{j}^{n} a_{n,2}^{(j)}\right] \\
& \leq \frac{1}{2} \sup _{ \|a\|_2\leq C}\left[u_{n-1}(a)+L\sum_{j=1}^p s_{j}^{n} a^{(j)}\right]+\frac{1}{2} \sup _{ \|a\|_2\leq C}\left[u_{n-1}(a)-L\sum_{j=1}^p s_{j}^{n} a^{(j)}\right] \\
& =\mathbb{E}_{r_n}\left[\sup _{ \|a\|_2\leq C} u_{n-1}(a)+r_nL \sum_{j=1}^ps_{j}^{n}a^{(j)}\right] .
\end{align*}
Therefore we proved
\begin{align}
(1-\epsilon) \mathbb{E}_{r}\left[\sup _{ \|a\|_2\leq C} u_{n}(a)\right]\leq\mathbb{E}_{r}\left[\sup _{ \|a\|_2\leq C} u_{n-1}(a)+r_nL \sum_{j=1}^ps_{j}^{n}a^{(j)}\right].\label{radecondition}
\end{align}
Now we can proceed similar for $u_{n-1}$. However since $s_{j}^{n}$ depends on $r_{n-1}$ we need to condition on $r_{n-1}$.
Again by definition of the supremum there exist $\|a_{n-1,1}\|_2, \|a_{n-1,2}\|_2\leq C $ such that

\begin{align*}
& \mathbb{E}_r\left[u_{n-2}\left(a_{n-1,1}\right)+r_nL \sum_{j=1}^ps_{j}^{n}a_{n-1,1}^{(j)}+f(a_{n-1,1},z_{n-1}) \Biggl|  r_{n-1}=1\right]\\
&\geq(1-\epsilon)\mathbb{E}_r\left[\sup _{ \|a\|_2\leq C} u_{n-2}(a)+r_nL \sum_{j=1}^ps_{j}^{n}a^{(j)}+f\left(a,z_{n-1}\right)\Biggl|  r_{n-1}=1\right] \\
\text { and } & \mathbb{E}_r\left[u_{n-2}\left(a_{n-1,1}\right)+r_nL \sum_{j=1}^ps_{j}^{n}a_{n-1,1}^{(j)}-f(a_{n-1,1},z_{n-1}) \Biggl|  r_{n-1}=-1\right]\\
&\geq(1-\epsilon)\mathbb{E}_r\left[\sup _{ \|a\|_2\leq C} u_{n-2}(a)+r_nL \sum_{j=1}^ps_{j}^{n}a^{(j)}-f\left(a,z_{n-1}\right)\Biggl|  r_{n-1}=-1\right].
\end{align*}
(Note: Since we conditioned on $r_{n-1}$ , we have that the choice  $a_{n-2}$ and therefore also $s_{j}^{n-1} :=\operatorname{sgn}\left(a_{n-1,1}^{(j)}-a_{n-1,2}^{(j)}\right)$ is independent of $r_{n-1}$. This will be important for the last step. )
Thus by definition of $\mathbb{E}_{r}$ and the Lipschitz assumption we further obtain for \eqref{radecondition},
\begin{align*}
& (1-\epsilon) \mathbb{E}_r\left[\sup _{ \|a\|_2\leq C} u_{n-1}(a)+r_nL \sum_{j=1}^ps_{j}^{n}a^{(j)}\right] \\
&= (1-\epsilon) \mathbb{E}_r\left[\sup _{ \|a\|_2\leq C} u_{n-2}(a)+r_nL \sum_{j=1}^ps_{j}^{n}a^{(j)}+r_{n-1}f\left(a,z_{n-1}\right)\right] \\
&\leq \frac{1}{2}\mathbb{E}_r\left[u_{n-2}\left(a_{n-1,1}\right)+r_nL \sum_{j=1}^ps_{j}^{n}a_{n-1,1}^{(j)}+f(a_{n-1,1},z_{n-1}) \Biggl|  r_{n-1}=1\right]+\\
&\quad \,\,\frac{1}{2}\mathbb{E}_r\left[u_{n-2}\left(a_{n-1,1}\right)+r_nL \sum_{j=1}^ps_{j}^{n}a_{n-1,1}^{(j)}-f(a_{n-1,1},z_{n-1}) \Biggl|  r_{n-1}=-1\right]\\
&\leq \frac{1}{2}\mathbb{E}_r\left[u_{n-2}\left(a_{n-1,1}\right)+r_nL \sum_{j=1}^ps_{j}^{n}a_{n-1,1}^{(j)}+L \sum_{j=1}^ps_{j}^{n-1}a_{n-1,1}^{(j)} \Biggl|  r_{n-1}=1\right]+\\
&\quad \,\,\frac{1}{2}\mathbb{E}_r\left[u_{n-2}\left(a_{n-1,1}\right)+r_nL \sum_{j=1}^ps_{j}^{n}a_{n-1,1}^{(j)}-L \sum_{j=1}^ps_{j}^{n-1}a_{n-1,1}^{(j)} \Biggl|  r_{n-1}=-1\right].
\end{align*}

Taking the supremum over $a$ we further obtain
\begin{align*}
&\leq \frac{1}{2}\mathbb{E}_r\left[\sup _{ \|a\|_2\leq C}u_{n-2}\left(a\right)+r_nL \sum_{j=1}^ps_{j}^{n}a^{(j)}+L \sum_{j=1}^ps_{j}^{n-1}a^{(j)} \Biggl|  r_{n-1}=1\right]+\\
&\quad \,\,\frac{1}{2}\mathbb{E}_r\left[\sup _{ \|a\|_2\leq C}u_{n-2}\left(a\right)+r_nL \sum_{j=1}^ps_{j}^{n}a^{(j)}-L \sum_{j=1}^ps_{j}^{n-1}a^{(j)} \Biggl|  r_{n-1}=-1\right]\\
&=\mathbb{E}_r\left[\sup _{ \|a\|_2\leq C}u_{n-2}\left(a\right)+r_nL \sum_{j=1}^ps_{j}^{n}a^{(j)}+r_{n-1}L \sum_{j=1}^ps_{j}^{n-1}a^{(j)} \right].
\end{align*}

Proceeding in the same way for all other $u_i$ with  $i\leq n-2$ leads to 
\begin{align*}
(1-\epsilon)^n\frac{1}{n} \mathbb{E}_r\left[\sup _{ \|a\|_2\leq C} \sum_{i=1}^n f\left(a,z_i\right)r_i\right]&\leq \frac{1}{n} \mathbb{E}_r\left[\sup _{ \|a\|_2\leq C} \sum_{i=1}^n r_iL\sum_{j=1}^p s_{j}^{i} a^{(j)}\right]\\
&=\frac{L}{n} \mathbb{E}_r\left[\sup _{ \|a\|_2\leq C} \sum_{j=1}^pa^{(j)}\sum_{i=1}^n r_is_{j}^{i}\right]\\
&\leq \frac{LC}{n} \mathbb{E}_r\left[\sqrt{ \sum_{j=1}^p\left(\sum_{i=1}^n r_i s_{j}^{i} \right)^2}\right]\\
&\leq \frac{LC}{n}\sqrt{ \sum_{j=1}^p \mathbb{E}_r\left[\left(\sum_{i=1}^n r_i s_{j}^{i}\right)^2\right]}.
\end{align*}
Because the sum over $j$ is now outside of the expectation we can get rid of $s_{j}^{i}$ by definition of $r_i$ and by the fact that $s_{j}^{i}$ is independent of $r_i$. Therefore we obtain for $\epsilon \rightarrow 0$
$$\frac{1}{n} \mathbb{E}_r\left[\sup _{ \|a\|_2\leq C} \sum_{i=1}^n f\left(a,z_i\right)r_i\right]\leq \frac{LC}{n} \sqrt{ p\mathbb{E}_r\left[\left(\sum_{i=1}^n r_i\right)^2\right]} = LC\sqrt{\frac{p}{n}}.$$
\end{proof}

\vspace{0.3cm}

\begin{corollary}
\label{cor:Rademacher-Gclass}
Let the Assumptions of Proposition 
\ref{prop:Rademacher-Decomposition}  and Proposition \ref{prop:Rademacher-Lipschitz} be satisfied with $R_1=R_2=R$. 
Suppose that $||h||_\infty < \infty$. Then 
\[  \widehat{\Re}_n(\cG ) \leq  16||h||_\infty c_\sigma \;   \frac{ R^2(d+1)^2}{\sqrt{M}}\;  \sqrt{\frac{d+1}{n}}  \;. \]
\end{corollary}

\begin{proof}[Proof of Corollary \ref{cor:Rademacher-Gclass}]
the results follows immediately by combining Proposition 
\ref{prop:Rademacher-Decomposition}  and Proposition \ref{prop:Rademacher-Lipschitz}. 
Indeed, since $\sigma^{'}$ and $\sigma^{''}$ are $c_\sigma$-Lipschitz, we obtain  
for any $g \in \cG_1^{(j)}$, 
$j \in [d+1]$
\[ |g(b_1, z) - g(b_2, z)| \leq 4||h||_\infty c_\sigma\; ||b_1 - b_2||_2 \;,  \]
for all $z \in \cZ$ with $||x|| \leq 2$,  $b_1, b_2 \in \mbr^{d+1}$ with $||b_j||_2 \leq R$. 
Similarly, for all  $g \in \cG_2^{(j,k)}$, $j,k \in [d+1]$, 
\[ |g(b_1, z) - g(b_2, z)| \leq 8||h||_\infty c_\sigma\; ||b_1 - b_2||_2 \;.  \]
Hence, by Proposition \ref{prop:Rademacher-Lipschitz} 
\[ \widehat{\Re}_n(\cG_1^{(j)} ) \leq 4||h||_\infty c_\sigma\;R\; \sqrt{\frac{d+1}{n}}, 
   \quad \widehat{\Re}_n(\cG_2^{(j, k)} ) \leq  8||h||_\infty c_\sigma\;R\; \sqrt{\frac{d+1}{n}} \;.\]
As a result, by Proposition \ref{prop:Rademacher-Decomposition}
\begin{align*}
  \widehat{\Re}_n(\cG ) &\leq  4||h||_\infty c_\sigma\; \frac{ R^2(d+1)}{\sqrt{M}}\;  \sqrt{\frac{d+1}{n}} +  
        8||h||_\infty c_\sigma\;\frac{ R^2(d+1)^2}{\sqrt{M}}\;  \sqrt{\frac{d+1}{n}} \\
&\leq 16||h||_\infty c_\sigma  \frac{ R^2(d+1)^2}{\sqrt{M}}\;  \sqrt{\frac{d+1}{n}}  \;. 
\end{align*}
\end{proof}


\subsection{Elementary Inequalities}

\begin{lemma}
\label{prop0}
The function $f(x)=(1-x)^jx^r$ for $j,r >0$ has a global maximum 
on $[0,1]$ at $\frac{r}{r+j}$. 
Therefore we have $\sup_{x\in[0,1]}f(x)\leq \left(\frac{r}{r+j}\right)^r$.
\end{lemma}

\begin{proof}[Proof of Lemma \ref{prop0}]
If $f(x)$ is continuous and has an maximum, then $\log f(x)$ will at the same point because $\log x$ is a monotonically increasing function.
\begin{align*} \ln f(x) & = r \log x + j\log(1 - x) \\
\Rightarrow \frac{\operatorname{d} \ln f(x)}{\operatorname{d} x} & = \frac{r}{x} - \frac{j}{1-x} \\
\Rightarrow \frac{\operatorname{d}^2 \ln f(x)}{\operatorname{d} x^2} & = -\frac{r}{x^2} - \frac{j}{(1-x)^2}.
\end{align*}
The second derivative is obviously less than zero everywhere, and the first derivative is equal zero at $\frac{r}{r+j}$. Therefore it is a global maximum.
\end{proof}

\vspace{0.3cm}

\begin{lemma}
\label{prop1}
Let $v\geq0, t \in \mathbb{N}$. Then
\[
\sum_{i=1}^t i^{-v} \leq \eta_v(t)\,\,\,\,\,\, \text{ with } \,\,\eta_v(t)\coloneqq
\begin{cases}
\frac{v}{v-1} &v>1\\
1+\log(t) & v=1\\
\frac{t^{1-v}}{1-v}& v\in[0,1[ 
\end{cases}
\]
\end{lemma}

\begin{proof}[Proof of Lemma \ref{prop1}]
 We will only prove the first case $v>1$. The other cases can be bounded with the same arguments.
 Since
$$
 i^{-v} \leq \int_{i-1}^{i} x^{-v} d x
$$
we have for $v>1$,
$$
\sum_{i=1}^{t} i^{-v} \leq 1+\int_{1}^{t} x^{-v} d x \leq \frac{v}{v-1} .
$$
\end{proof}

\vspace{0.3cm}

\begin{lemma}
\label{lem:calcs}
Let $a\geq b\geq 0$. Then we have
\begin{align*}
\sum_{s=1}^{T-1} \frac{1}{s^a}\frac{1}{(T-s)^b} \leq \frac{2^a}{T^a}\eta_b(T-1) + \frac{2^b}{T^b}\eta_a(T-1).\\
\end{align*}

\end{lemma}

\begin{proof}[Proof of Lemme \ref{lem:calcs}]
For any $t\in [T]$ we have
\begin{align*}
\sum_{s=1}^{T-1}  \frac{1}{s^a}\frac{1}{(T-s)^b} &= \sum_{s=t+1}^{T-1} \frac{1}{s^a}\frac{1}{(T-s)^b} +\sum_{s=1}^t \frac{1}{s^a}\frac{1}{(T-s)^b} \\
&\leq \frac{1}{t^a}\sum_{s=t+1}^{T-1} \frac{1}{(T-s)^b} + \frac{1}{(T-t)^b}\sum_{s=1}^t \frac{1}{s^a}\\
&\leq \frac{1}{t^a}\sum_{s=1}^{T-1} \frac{1}{s^b} + \frac{1}{(T-t)^b}\sum_{s=1}^{T-1} \frac{1}{s^a}\\
&\leq \frac{1}{t^a}\eta_b(T-1) + \frac{1}{(T-t)^b}\eta_a(T-1).\\
\end{align*}
Setting $t=T/2$ the result follows. 
\end{proof}

\vspace{0.3cm}

\begin{proposition}
\label{conditioning}
Let  $E_i$ be events with probability at least $1-\delta_i$ and set
$$
E:=\bigcap^k_{i=1}E_i
$$
If we can show for some event $A$ that $\mathbb{P}(A|E)\geq 1-\delta$  then we also have
\begin{align*}
\mathbb{P}(A)&\geq\int_{E}\mathbb{P}(A|\omega)d\mathbb{P}(\omega)
\geq (1-\delta)\mathbb{P}(E)\\
&=(1-\delta)\left(1-\mathbb{P}\left(\bigcup_{i=1}^k (\Omega/E_i)\right)\right)\geq(1-\delta)\left(1-\sum_{i=1}^k\delta_i\right).
\end{align*}
\end{proposition}



\subsection{Concentration Inequalities}

The following concentration result for Hilbert space valued random variables can be found in \cite{Caponetto}.

\vspace{0.2cm}

\begin{proposition}[Bernstein Inequality]
\label{concentrationineq0}
Let $W_{1}, \cdots, W_{n}$ be i.i.d random variables in a separable Hilbert space $\cH$ with norm $\|\cdot\|_{\cH}$. 
Suppose that there are two positive constants $B$ and $V$ such that
\begin{align}
\label{cons}
\mathbb{E}\left[\left\|W_{1}-\mathbb{E}\left[W_{1}\right]\right\|_{\cH}^{l}\right] \leq \frac{1}{2} l ! B^{l-2} V^{2}, 
\quad \forall l \geq 2 \;. 
\end{align}
Then for any $\delta \in (0,1]$, the following holds with probability at least $1-\delta$:
$$
\left\|\frac{1}{n} \sum_{k=1}^{n} W_{k}-\mathbb{E}\left[W_{1}\right]\right\|_{\cH} 
\leq 2\left(\frac{B}{n}+\frac{V}{\sqrt{n}}\right) \log \left(\frac{4}{\delta} \right)\;  .
$$
In particular, \eqref{cons} holds if
$$
\left\|W_{1}\right\|_{\cH} \leq B / 2 \quad \text { a.s., } \quad \text { and } 
\quad \mathbb{E}\left[\left\|W_{1}\right\|_{\cH}^{2}\right] \leq V^{2} \;.
$$
\end{proposition}


\subsection{Operator Inequalities}


\begin{proposition}
\label{Opbound1}
\begin{itemize}
\item[]
\item[a)]For any $\lambda>0$ we have
$$
\left\|\widehat{\mathcal{C}}^{-\frac{1}{2}}_{M,\lambda}\widehat{\mathcal{Z}}^*_M\right\|\leq1.
$$
\item[b)]  For any $\lambda>0$ we have
$$
\|\mathcal{C}_{M,\lambda}^{-1}\mathcal{Z}^*_M\mathcal{L}_{M,\lambda}^{\frac{1}{2}}\|\leq 2.
$$

\end{itemize}

\end{proposition}

\begin{proof}
The bound of $b)$ can be found in \cite{rudi2017generalization} (see proof of lemma 3).
For $a)$ we have 
\begin{align*}
\left\|\widehat{\mathcal{C}}^{-\frac{1}{2}}_{M,\lambda}\widehat{\mathcal{Z}}^*_M\right\|^2&=
\left\|(\widehat{\mathcal{Z}}_M^*\widehat{\mathcal{Z}}_M+\lambda)^{-1/2}\widehat{\mathcal{Z}}_M^*\right\|^2\\
&=\left\|(\widehat{\mathcal{Z}}_M^*\widehat{\mathcal{Z}}_M+\lambda)^{-1/2}\widehat{\mathcal{Z}}_M^*\widehat{\mathcal{Z}}_M(\widehat{\mathcal{Z}}_M^*\widehat{\mathcal{Z}}_M+\lambda)^{-1/2}\right\|\\
&=\left\|\widehat{\mathcal{Z}}_M^*\widehat{\mathcal{Z}}_M(\widehat{\mathcal{Z}}_M^*\widehat{\mathcal{Z}}_M+\lambda)^{-1}\right\|\leq1.
\end{align*}
\end{proof}


\vspace{0.3cm}

\begin{proposition}
\label{Opbound2}
We have for  any $\lambda>0$ such that 
$n\geq \frac{8\kappa^2 \tilde\beta}{\lambda}$ 
with 
$\tilde{\beta}:= \log \frac{4 \kappa^2(\left(1+2\log\frac{2}{\delta}\right)4\mathcal{N}_{\mathcal{L}_{\infty}}(\lambda)+1)}{\delta\|\mathcal{L}_\infty\|}$ 
and 
$M\geq \frac{8 (d+2)\kappa^2 \beta_\infty}{\lambda}\vee 8\kappa^4\|\mathcal{L}_\infty\|^{-1}\log^2 \frac{2}{\delta}$  
with 
$\beta_\infty=\log \frac{4 \kappa^2(\mathcal{N}_{\mathcal{L}_\infty}(\lambda)+1)}{\delta\|\mathcal{L}_\infty\|} $, 
that with probability at least $1-\delta$,
$$ 
\left\|\widehat{C}_{M,\lambda}^{-\frac{1}{2}}C_{M,\lambda}^{\frac{1}{2}}\right\| \leq 2, \quad  \left\|\widehat{C}_{M,\lambda}^{\frac{1}{2}}C_{M,\lambda}^{-\frac{1}{2}}\right\| \leq 2 .
$$
\end{proposition}

\begin{proof}[Proof of Proposition \ref{Opbound2}]
The proof follows exactly the same steps as in \cite{nguyen2023random}  (Proposition A.15.).
\end{proof}


\vspace{0.3cm}

\begin{proposition}[\cite{nguyen2023random} Proposition A.18.]
\label{Opbound3}
For any $\lam>0$ with $M\geq \frac{8 (d+2)\kappa^2 \beta_\infty}{\lam}$ 
we have with probability at least $1-\delta$
$$
\cN_{\cL_{M}}(\lam)
\leq  4\left(1+2\log\frac{2}{\delta}\right)\mathcal{N}_{\cL_{\infty}}(\lam).
$$
\end{proposition}


\vspace{0.3cm}

\begin{proposition}
\label{Opbound4}
For any $\lambda>0$ the following event hold true with probability at least $1-\delta$,
\begin{align}
\left\|C_{M, \lambda}^{-1 / 2}\left(C_M-\widehat{C}_M\right)\right\| \leq \left(\frac{2\kappa}{\sqrt{\lambda} n}+\sqrt{\frac{ 4\kappa^2 \mathcal{N}_{\mathcal{L}_M}(\lambda) }{ n}}\right)\log \frac{2}{\delta}.\label{cond1}
\end{align}

\end{proposition}

\begin{proof}
The proof follows exactly the same steps as in  \cite{nguyen2023random} (see Proposition A.21. (E3)).
\end{proof}


\vspace{0.3cm}

\begin{proposition}
\label{Opbound5}
We have for  any $\lambda>0$ such that 
$n\geq \frac{8\kappa^2 (\log\frac{6}{\delta}+\tilde\beta)}{\lambda}$ with 
\[ \tilde{\beta}:= \log \frac{4 \kappa^2(\left(1+2\log\frac{2}{\delta}\right)4\mathcal{N}_{\mathcal{L}_{\infty}}(\lambda)+1)}{\delta\|\mathcal{L}_\infty\|} \] 
and 
$M\geq \frac{8 (d+2)\kappa^2 \beta_\infty}{\lambda}\vee 8\kappa^4\|\mathcal{L}_\infty\|^{-1}\log^2 \frac{2}{\delta}$  
with 
$\beta_\infty:=\log \frac{4 \kappa^2(\mathcal{N}_{\mathcal{L}_\infty}(\lambda)+1)}{\delta\|\mathcal{L}_\infty\|} $, 
that with probability at least $1-\delta$,
\begin{align*}
\left\|\widehat C_{M, \lambda}^{-1}C_{M,\lambda}\right\| \leq 
\frac{3}{2} + 
\sqrt{\frac{ 192 \kappa^2 \mathcal{N}_{\mathcal{L}_{\infty}}(\lambda)\log\frac{6}{\delta} }{ \lambda n}}
\log\left( \frac{6}{\delta}\right) \; .
\end{align*}
If we additionally assume that  $n\geq \frac{768 \kappa^2 \mathcal{N}_{\mathcal{L}_{\infty}}(\lambda)\log^3 \frac{6}{\delta}}{\lambda}$ we obtain with probability at least $1-\delta$,
\begin{align*}
\left\|\widehat C_{M, \lambda}^{-1}C_{M,\lambda}\right\| \leq 2.
\end{align*}
\end{proposition}

\begin{proof}
Using Proposition \ref{Opbound4} we obtain with probability at least $1-\delta$,
\begin{align*}
&\left\|\widehat{C}_{M,\lambda}^{-1} C_{M,\lambda}\right\|\\
&\leq\left\|\widehat{C}_{M,\lambda}^{-1}\left(\widehat{C}_{M}-C_{M}\right)\right\|_{HS}+1\\
&\leq\frac{1}{\sqrt{\lambda}} \left\|\widehat{C}_{M,\lambda}^{-\frac{1}{2}} C_{M,\lambda}^{\frac{1}{2}}\right\|\left\|C_{M,\lambda}^{-\frac{1}{2}}\left(\widehat{C}_{M}-C_{M}\right)\right\|_{HS}+1\\
&\leq \frac{1}{\sqrt{\lambda}} \left\|\widehat{C}_{M,\lambda}^{-\frac{1}{2}} C_{M,\lambda}^{\frac{1}{2}}\right\| \left(\frac{2\kappa}{\sqrt{\lambda} n}+\sqrt{\frac{ 4\kappa^2 \mathcal{N}_{\mathcal{L}_M}(\lambda) }{ n}}\right)\log \frac{2}{\delta}+1.
\end{align*}
From Proposition \ref{Opbound2} we have with probability at least $1-\delta$,
\begin{align}
\left\|\widehat{C}_{M,\lambda}^{-\frac{1}{2}}C_{M,\lambda}^{\frac{1}{2}}\right\| \leq 2\label{cond2}
\end{align}

and therefore
\begin{align}
\left\|\widehat{C}_{M,\lambda}^{-1} C_{M,\lambda}^{1}\right\|\leq \frac{2}{\sqrt{\lambda}} \left(\frac{2\kappa}{\sqrt{\lambda} n}+\sqrt{\frac{ 4\kappa^2 \mathcal{N}_{\mathcal{L}_M}(\lambda) }{ n}}\right)\log \frac{2}{\delta}.\label{ineqT2ii}
\end{align}

From Proposition \ref{Opbound3} we have
\begin{align}
\mathcal{N}_{\mathcal{L}_{M}}(\lambda)\leq  \left(1+2\log\frac{2}{\delta}\right)4\mathcal{N}_{\mathcal{L}_{\infty}}(\lambda).\label{cond3}
\end{align}

Plugging this bound into \eqref{ineqT2ii} leads to
\begin{align}
\left\|\widehat{C}_{M,\lambda}^{-1} C_{M,\lambda}^{1}\right\|\leq \frac{2}{\sqrt{\lambda}} \left(\frac{2\kappa}{\sqrt{\lambda} n}+\sqrt{\frac{ 16\kappa^2 \left(1+2\log\frac{2}{\delta}\right)\mathcal{N}_{\mathcal{L}_{\infty}}(\lambda) }{ n}}\right)\log \frac{2}{\delta}+1.\label{resultopbound}
\end{align}

The above inequality therefore holds if we condition on the events (\eqref{cond1},  \eqref{cond2}, \eqref{cond3}). 
Using Proposition \ref{conditioning} therefore shows that the above inequality \eqref{resultopbound} holds with probability at least $1-3\delta$ . Redefining $\delta=3\delta$ and using $\left(1+2\log\frac{6}{\delta}\right)\leq 3\log\frac{6}{\delta}$ proves that wit probability at least $1-\delta$,
\begin{align*}
\left\|\widehat C_{M, \lambda}^{-1}C_{M,\lambda}\right\|\leq 1+ \left(\frac{4\kappa}{\lambda n}+\sqrt{\frac{ 192 \kappa^2 \mathcal{N}_{\mathcal{L}_{\infty}}(\lambda)\log\frac{6}{\delta} }{ \lambda n}}\right)\log \frac{6}{\delta}.
\end{align*}
In the last step we use that $n\geq \frac{8\kappa^2 \log\frac{6}{\delta}}{\lambda}$ to obtain 
\begin{align*}
\left\|\widehat C_{M, \lambda}^{-1}C_{M,\lambda}\right\| \leq \frac{3}{2} + \sqrt{\frac{ 192 \kappa^2 \mathcal{N}_{\mathcal{L}_{\infty}}(\lambda)\log\frac{6}{\delta} }{ \lambda n}}\log \frac{6}{\delta}.
\end{align*}

If we additionally assume that $n\geq \frac{768 \kappa^2 \mathcal{N}_{\mathcal{L}_{\infty}}(\lambda)\log^3 \frac{6}{\delta}}{\lambda}$ we obtain that the last term is bounded by $2$. 
\end{proof}


\vspace{0.3cm}

\begin{proposition}
\label{prop:intermediate}
Let $M\geq \frac{8 (d+2)\kappa^2 \beta_\infty(\lambda)}{\lambda}$. 
We have with probability at least $1-\delta$ that the following event holds true:
\[
\left\|\mathcal{C}_{M, \lambda}^{-1 / 2}\left(\widehat{\mathcal{Z}}_M^* y-\mathcal{Z}_M^* g_\rho\right)\right\|
 \leq 
2\left(\frac{\kappa}{\sqrt{\lam} n}+ C\cdot \sqrt{\frac{\log(4/\delta) \cN_{\cL_{\infty}}(\lam)}{n}}\right) \log \left(\frac{4}{\delta}\right) \;.
\]
\end{proposition}

\begin{proof}[Proof of Proposition \ref{prop:intermediate}]
This follows from \cite[Lemma 6]{rudi2017generalization} and Proposition \ref{Opbound3} with 
\[
\mathcal{N}_{\mathcal{L}_{M}}(\lambda)\leq 
 \left(1+2\log\frac{4}{\delta}\right)4\mathcal{N}_{\mathcal{L}_{\infty}}(\lambda) 
 \leq 12 \log(4/\delta)\;\cN_{\cL_{\infty}}(\lam) \;. 
\]
\end{proof}

\end{document}